%% file: main.tex
\documentclass[lettersize,journal]{IEEEtran}

\usepackage{amsmath,amsfonts}
\usepackage{amsthm}
\usepackage{algorithm2e}

\usepackage{algorithm} 
\usepackage{array}
\usepackage{url}
\usepackage{verbatim}
\usepackage{graphicx}
\usepackage{cite}
\usepackage{hyperref}
\usepackage{breqn}
\usepackage{booktabs}
\usepackage{multirow}

\hypersetup{hidelinks}

\theoremstyle{plain}
\newtheorem{theorem}{Theorem}[section]
\newtheorem{proposition}[theorem]{Proposition}

\usepackage[normalem]{ulem}
\usepackage{xcolor}

\newcommand{\del}[1]{}
\newcommand{\add}[1]{\textcolor{black}{#1}}
\newcommand{\rev}[2]{\del{#1}\add{#2}}

\newcolumntype{C}[1]{>{\centering\arraybackslash}p{#1}}
\allowdisplaybreaks
\begin{document}
\bstctlcite{BSTcontrol} 

\setlength{\textfloatsep}{0.75\textfloatsep}
\setlength{\dbltextfloatsep}{0.75\dbltextfloatsep}
\setlength{\belowcaptionskip}{0pt}

\title{Censoring-Aware In-Context Learning for Generalized Supplier Lead Time Estimation in Supply Chain Planning}

\author{
Christopher Wang,
Sebastien Ouellet,
Behrouz Haji Soleimani,
and Ali Etemad
\thanks{Christopher Wang and Ali Etemad are with the Department of Electrical and Computer Engineering, Queen's University, Kingston, ON K7L 3N6, Canada.}
\thanks{Christopher Wang, Sebastien Ouellet, and Behrouz Haji Soleimani are with Kinaxis Inc., Ottawa, ON K2K 3J1, Canada.}
\thanks{Corresponding author: Christopher Wang \href{mailto:15cw62@queensu.com}{15cw62@queensu.ca}}
}

\markboth{IEEE Transactions on Automation Science and Engineering, 2026}
{Wang \MakeLowercase{\textit{et al.}}: Censoring-Aware In-Context Learning}
\maketitle

\begin{abstract}
Supplier lead time forecasting is a central input to material requirements planning, inventory optimization, and supply chain risk management. \rev{However, many industrial lead time datasets are naturally right-censored: at the time forecasts are required, some orders have not yet arrived. Standard regression and classification approaches discard this information, while conventional survival models require task-specific modeling.} {Industrial lead time data are naturally right-censored: when forecasts are needed, some orders have not yet arrived. Regression and classification approaches discard this information.} We propose LeadTime-ICL (LT-ICL), a censoring-aware in-context learning model for probabilistic lead time forecasting. LT-ICL combines a transformer backbone with a conditional normalizing-flow head, producing a full predictive distribution over lead times. The model is pretrained on synthetic right-censored lead time tasks, enabling in-context adaptation to new industrial datasets without task-specific parameter updates. \rev{We provide theoretical support for this formulation by showing that excess CRPS is bounded by prior misspecification and amortized approximation errors, providing clear direction for improving forecasting performance.}{A general risk decomposition bounds excess CRPS by prior misspecification and amortized approximation errors, identifying two levers for improvement.} We evaluate LT-ICL on 24 proprietary supply-chain datasets spanning seven industries \add{that share one input schema, with censoring introduced retrospectively through historical cutoff dates}. LT-ICL achieves the lowest point-forecasting error on 13 of the 24 datasets, and the lowest probabilistic forecasting error on 12 datasets, yielding the best average rank across both. These results support right-censored probabilistic forecasting as a practical formulation for supplier lead time prediction and demonstrate that pretrained in-context models can provide accurate, low-adaptation-cost forecasting for industrial planning systems.
\end{abstract}

\def\abstractname{Note to Practitioners}
\begin{abstract}
This paper addresses the challenge of estimating supplier lead times from historical procurement data. Existing approaches often face two practical limitations: they must be retrained for each new firm or industry setting, and they typically cannot make direct use of open orders that have been placed but not yet received. We propose LT-ICL, a censoring-aware in-context learning method for supplier lead time prediction. LT-ICL can be redeployed across firms, industries, and datasets without task-specific retraining while incorporating open orders as partial observations of lead time. For practitioners, the main value of LT-ICL is that it shifts model adaptation from repeated retraining to amortized pretraining. When organizations can afford an upfront training cost, LT-ICL provides a reusable forecasting model that can be applied to new lead time datasets with low deployment overhead. The theoretical analysis also identifies two principled levers for practitioners to improve LT-ICL for their use cases. First, practitioners can reduce prior misspecification, for example by incorporating domain knowledge into the synthetic pretraining process. Second, practitioners can reduce approximation error, for example by increasing model capacity, training for longer, or improving the pretraining procedure. LT-ICL can therefore support the development of reusable lead time forecasting systems that efficiently adapt across industrial datasets. \add{Two limitations should be noted. The evaluation retrospectively recreates open orders from completed orders, and excludes orders unresolved at extraction. All datasets were harmonized to a shared input schema; further extension is needed for LT-ICL to exploit additional customer-specific fields.}
\end{abstract}


\input{_01_introduction}
\input{_02_related_work}
\input{_03_methodology}
\input{_04_case_study}
\input{_05_results}
\input{_06_conclusion}

\bibliographystyle{IEEEtran}
\bibliography{references}

\end{document}

%% file: _01_introduction.tex
\section{Introduction}\label{introduction}
The proper specification of lead times has long been a challenge in supply chain management. In Material Requirements Planning (MRP) systems, planned lead times are critical inputs, with inaccurate lead time settings leading to degraded planning outcomes \cite{MARLIN1986179, Enns01012001}. Furthermore, lead time estimation has important implications outside of planning. Upstream, it contributes to the bullwhip effect, whereby variability in lead time estimates amplifies demand variability as it propagates through the supply chain \cite{michna2015impactleadtimeforecasting, Disney09102025}. Downstream, it induces the lead time syndrome, whereby lead time revisions increase variability in production order releases and can create temporary bottlenecks \cite{SELCUK2006427, Selcuk01052009}. These challenges are further compounded by planners often specifying conservatively long planned lead times as a buffer against uncertainty \cite{Graves2011}. Lead times also carry important contractual implications, as changes in supplier lead times are among the most common sources of buyer--supplier grievances \cite{DAS2003171}. Together, these consequences have motivated growing interest in data-driven approaches for lead time estimation.

Recent surveys on lead time estimation \cite{9143119, KOBLASA2026111713, MULLER2025100130} document the increasing adoption of learning methods. However, industrial lead time estimation poses two challenges for the design of general prediction methods. First, lead time datasets vary substantially across materials, suppliers, firms, and industries. Although different organizations may share relevant characteristics, such as common suppliers, lead time data is often proprietary. This makes learning a model that generalizes across industries challenging \cite{zheng2023federated}. Prior work has therefore focused on methodologies and evaluations tailored to individual firms and industries \cite{9143119, MULLER2025100130}. Second, at any given forecasting date, some orders are \textit{open}: their final lead times are unobserved, but their elapsed durations since being placed provide right-censored information. Prior work has largely discarded this information by formulating lead time estimation as regression \cite{STEINBERG2023100003, pom}, time-series forecasting \cite{gabellini, MATTHES2025807}, or classification via late-order prediction \cite{Brintrup02062020, ZAGHDOUDI2024110590, SHIDPOUR2025100172}. Each elapsed duration provides a useful lower bound on the eventual lead time, which to our knowledge has not been extensively explored in lead time estimation.  

To address these gaps, we propose LeadTime-ICL (LT-ICL), an in-context learning method for censoring-aware lead time prediction. LT-ICL combines a transformer encoder \cite{vaswani2023attentionneed} with a conditional normalizing-flow head \cite{rezende2016variationalinferencenormalizingflows} to produce flexible predictive distributions over lead times. Right-censored observations are encoded using elapsed durations and censoring indicators, allowing open orders to contribute partial information during prediction. LT-ICL is pretrained on synthetic right-censored tasks sampled from a latent hierarchical semi-parametric prior designed to capture key properties of industrial lead time datasets. Because pretraining uses only synthetic data, the method does not require sharing proprietary training data and can be applied across industries without task-specific retraining. We validate LT-ICL on 24 real industrial lead time datasets spanning multiple industries and benchmark it against classical machine learning, deep learning, and other in-context learning (ICL) regression- and survival-analysis methods for both point and probabilistic forecasting.  

This paper makes the following three contributions:
\begin{enumerate}
    \item We formulate supplier lead time prediction as a right-censored probabilistic estimation problem, allowing open orders to contribute information through their elapsed durations rather than treating them as missing or unusable.
    \item We introduce LT-ICL, a censoring-aware in-context learning model pretrained on synthetic data that produces full predictive distributions over lead times. Across 24 real industrial lead time datasets, LT-ICL achieves the best overall performance among the evaluated methods for both point and probabilistic forecasting.
    \item We derive a forecasting-risk decomposition that separates the error of \rev{LT_ICL} {ICL methods} into prior misspecification and amortized approximation error. This identifies two principal sources of forecasting error and provides a theoretical basis for improving \rev{LT_ICL}{ICL methods} through richer synthetic priors and more accurate inference architectures. 
\end{enumerate}

%% file: _02_related_work.tex
\section{Related Work}\label{sec:related_work}
\subsection{Lead Time Estimation}
Tatsiopoulos and Kingsman classify two broad treatments of dynamic lead times in supply chain settings: endogenous and exogenous \cite{TATSIOPOULOS1983351}. In endogenous formulations, lead times are treated as decision variables that can be influenced through production planning and control \cite{MILNE2015220, JANSEN2019585}. By contrast, many operational planning models require lead times as exogenous inputs that must be specified in advance \cite{Graves2011}. Within this exogenous setting, Schneckenreither et al. further distinguished reactive, proactive, and predictive lead time strategies \cite{Schneckenreither03062021}. Our work belongs to the exogenously specified, predictive lead time setting, where lead times are estimated from historical data rather than adjusted through direct process control. 

Lead time estimation is most commonly formulated as a supervised regression problem with a continuous target. Tree-based ensembles (both bagging and boosting) are often strong empirical baselines, including in supplier lead time estimation for safety-stock and inventory planning \cite{BARROS2023106671,pom}, delivery-date prediction before detailed process planning \cite{s10845-023-02290-2, 11362374}, and product-availability prediction under disruption \cite{CAMUR2024123226}. However, these methods generally treat lead time prediction as a fully observed supervised learning problem and do not directly exploit partially observed open orders during training or evaluation.

Neural models have also been shown to produce strong results, such as multilayer perceptrons (MLP) outperforming competing methods in a simulated order-release planning setting \cite{Schneckenreither03062021}. Recurrent architectures such as long short-term memory (LSTM) models have previously been applied as forecasters for lead times in both the automotive and coal industries \cite{gabellini, 11468209}. However, standard time-series formulations are not always well suited to transactional lead time data, where many concurrent orders may be placed at irregular times and responses occur at the order level rather than at regular time steps. Closest to our work, Matthes et al. compare transformer models with LSTM and NGBoost models for lead time prediction in the refining industry \cite{MATTHES2025807}. They find that transformers provide strong distributional coverage, though their evaluation is limited to a single firm, and their approach neither exploits censoring nor transfers to new datasets without retraining. In contrast, our goal is to learn a single censoring-aware model that can adapt to new lead time datasets without retraining.

\subsection{Censored and Distributional Prediction}
Open orders introduce a censoring structure into lead time prediction: the final lead time is unobserved, but the elapsed duration since order placement time provides a lower bound on the eventual lead time. This differs from standard missing-label prediction because open orders contain partial information that should not be discarded. Prior work has shown that lead time prediction improves as more procurement information becomes available for open orders \cite{STEINBERG2023100003}. However, this information is incorporated at discrete purchasing stages rather than through a general elapsed-duration censoring formulation.

Survival-analysis methods provide a natural statistical framework for censored outcomes. Cox proportional hazards models have previously been used to analyze lead times for metallic aerospace components with censored open-order information \cite{DeCosJuez01092012}. However, that study is primarily descriptive and provides limited comparison against modern predictive baselines. More recently, deep learning has been combined with censored information in supply-chain applications, but for time-to-survive forecasting of critical parts rather than supplier lead time prediction \cite{li2025end}. Our work instead treats supplier lead time estimation as a right-censored probabilistic prediction problem and evaluates both point and distributional forecasting performance across multiple real industrial datasets.

\add{Censoring has also been studied in supply chain demand forecasting, particularly under lost sales, where observed sales provide an incomplete representation of latent demand due to stock-outs. Prior approaches have addressed this challenge through Tobit Kalman filtering \cite{TRAPERO20241055}, survival analysis and doubly robust estimation \cite{tang2025offline}, and distributional learning from censored observations \cite{LIU2026}. While related, these approaches primarily focus on estimating latent demand or revenue distributions to support subsequent forecasting and decision-making. In contrast, our setting is concerned with the conditional distribution of individual order completion times, where an open order may itself be a prediction target. Consequently, in addition to informing the underlying lead time process, censoring provides order-specific predictive information through the elapsed duration, motivating our right-censored time-to-event formulation.}

\subsection{In-Context Learning and Prior-Data Fitted Networks}
In-context learning provides a mechanism for adapting to a new prediction task at inference time without updating model parameters by conditioning on examples from that task \cite{dong2024surveyincontextlearning}. Prior-data fitted networks formalize this idea for tabular data by pretraining a transformer model on synthetic tasks sampled from a prior over data-generating processes \cite{müller2024transformersbayesianinference}. This process has been shown to approximate the Bayesian posterior predictive distribution for new datasets under the synthetic data prior \cite{nagler2023statisticalfoundationspriordatafitted}. This approach has been especially influential in tabular prediction, where models such as TabPFNv2 use synthetic pretraining to perform fast supervised learning on small tabular datasets without task-specific retraining \cite{Hollmann2025}. \del{However, these types of models have not yet been extended to handle censored information.}

\add{
Recent work has extended tabular foundation models to survival analysis through two main approaches. The first adapts existing pretrained models to handle censored outcomes. TabPFN-CF reformulates survival prediction as a series of binary classification tasks \cite{kim2026tabularfoundationmodelssurvival}.  TabSA employs an accelerated failure time formulation with iterative imputation of censored observations \cite{vargasvieyra2026staying}. TabFM-MTLR trains a survival-specific multitask logistic regression prediction head on pretrained tabular representations \cite{10.1007/978-3-032-35387-0_23}. The second approach directly pretrains in-context learning models for survival prediction. Methods such as Survival In-Context (SIC), SurvivalPFN, and SurvPFN learn from synthetic censored tasks and perform inference on new datasets without task-specific retraining \cite{seletkov2026survivalincontextamortizedbayesian, qi2026survivalpfn, wang2026survpfn}.}

\add{
While these approaches primarily address general-purpose survival prediction, LT-ICL extends censoring-aware in-context learning to transactional supplier lead time forecasting. Specifically, LT-ICL introduces a hierarchical semi-parametric prior that captures procurement-specific structure, including part-site-supplier interactions, temporal effects, and heterogeneous lead time distributions. This prior is paired with a conditional normalizing-flow output to produce predictive lead time distributions, building on prior work using normalizing flows for censored survival prediction \cite{9564222}.
}

\del{Our method applies this principle to censored lead time prediction. Unlike conventional supervised lead time models, which must be retrained or tuned for each dataset, LT-ICL is pretrained once on synthetic censored lead time tasks and then applied directly to new industrial datasets through in-context learning. Unlike existing in-context learning methods, LT-ICL is able to natively use censored information. This connects supplier lead time estimation to the broader literature on amortized Bayesian prediction and in-context learning, while addressing a domain-specific challenge that existing PFN models do not handle directly: right-censored transactional outcomes from open orders.}

%% file: _03_methodology.tex
\section{Proposed Methodology}\label{sec:methodology}
\subsection{Problem Formulation}
\label{subsec:problem_formulation}
We consider the problem of forecasting supplier lead times from historical order data. Let $Y_i$ denote the true lead time of order $i$ with covariates $x_i \in \mathcal{X}$ and order date $d_i$. At a forecasting cutoff date $C$, orders can be divided into three categories: completed orders, whose arrival dates are observed; open orders, which have been placed but not yet received; and future orders, which have not yet been placed. Since $Y_i$ may be unobserved for open orders, we define the recorded duration
\begin{equation}
T_i = \min\{Y_i, C-d_i\},
\end{equation}
and the event indicator
\begin{equation}
\delta_i = \mathbb{I}\{Y_i \leq C-d_i\}.
\end{equation}
Thus, $\delta_i=1$ denotes a completed order for which $T_i=Y_i$, while $\delta_i=0$ denotes an open order for which $T_i=C-d_i$ and $Y_i>T_i$. A lead time dataset observed at cutoff $C$ is therefore
\begin{equation}
D=\{(x_i,T_i,\delta_i)\}_{i=1}^{n}.
\end{equation}
This formulation treats open orders as right-censored observations rather than missing labels.

Let $F_\theta(t\mid x)$, $f_\theta(t\mid x)$, and $S_\theta(t\mid x)=1-F_\theta(t\mid x)$ denote the conditional cumulative distribution, density, and survival function of lead times under parameters $\theta$. Under conditional independence assumptions, the likelihood factors as 
\begin{equation}
p(D \mid \theta) = \prod_{i=1}^{n} f_\theta(T_i\mid x_i)^{\delta_i} S_\theta(T_i\mid x_i)^{1-\delta_i}.
\end{equation}
Completed orders contribute through the density of the observed lead time, while open orders contribute through the probability of surviving beyond their elapsed duration. 

The forecasting objective is to produce a predictive distribution for a query order with covariates $x$. For a future order, this distribution is $p(Y\mid x,D)$. For an open order, the relevant predictive distribution is conditional on the order having already survived beyond $C-d$, $p(Y\mid Y>C-d, x,D)$. This distinction is important because an open order should not be forecast as if it were newly placed; its elapsed duration is already informative.

A Bayesian formulation places a prior $p(\theta)$ over possible lead time data-generating processes and obtains the posterior
\begin{equation}
p(\theta\mid D) \propto  p(D\mid \theta)p(\theta).
\end{equation}
The posterior predictive distribution for a future query order is then
\begin{equation}
p(Y\mid x,D) = \int p(Y\mid x,\theta)p(\theta\mid D)d\theta.
\end{equation}
For an open order query, the structure is identical with an added conditioning on $Y>C-d$. We drop this conditioning term for notational simplicity.

This posterior predictive distribution is the ideal target for probabilistic lead time forecasting: it incorporates completed orders, uses open orders through their censoring information, and accounts for uncertainty over the task-specific lead time process. However, exact posterior inference is generally intractable and would need to be repeatedly approximated for each new dataset. LT-ICL therefore learns this mapping directly. It replaces task-specific inference with a single pretrained in-context model that maps a censored context dataset and query covariates to a predictive lead time distribution  $q_{\phi}(Y \mid x, D) \approx p(Y \mid x, D)$.

\subsection{LT-ICL Architecture}
\label{subsec:architecture}
LT-ICL parameterizes the predictive distribution $q_{\phi}$ using an encoder-only transformer followed by a conditional normalizing-flow head. The model receives a context set of historical orders and a query set of orders to be forecast. The context set contains completed and open orders $D=\{(x_i,T_i,\delta_i)\}_{i=1}^{n}$, where $T_i$ is either the realized lead time or the elapsed censored duration. The query set contains open and future orders, represented by covariates $x_j$. \rev{Each context or query observation is converted into a token embedding. Numerical covariates are mapped to the model dimension through learned projection layers, while categorical covariates are mapped through learned embedding tables with fixed vocabularies. The response information is embedded separately. For completed context orders, the response embedding contains the observed lead time and event indicator $(T_i,\delta_i=1)$. For censored context orders, it contains the elapsed duration and event indicator $(T_i,\delta_i=0)$. For query orders, the response value and event indicator are masked. The final token representation is obtained by summing the covariate, response, censoring, and query-type embeddings.}{Each observation is embedded using learned projections for numerical covariates and fixed-vocabulary embeddings for categorical covariates. Part, site, and supplier identifiers carry no meaning across datasets and are reindexed within each task. Unseen query entities receive unused indices. Such entities are also represented during synthetic pretraining, to enable forecasting from identifiers without history. For context orders, the response embedding encodes the duration and event indicator $(T_i,\delta_i)$, where $T_i$ denotes the realized lead time if $\delta_i=1$ and elapsed duration if $\delta_i=0$. Both are masked for query orders. The final token sums the covariate, response, censoring, and query-type embeddings.}

The token sequence is processed by a transformer encoder with a PFN-style attention mask \cite{müller2024transformersbayesianinference}. Let $a_{ij}$ indicate whether token $i$ is permitted to attend to token $j$. We use the mask
\begin{equation}
a_{ij}
=
\begin{cases}
1, & \text{if token } j \text{ is a context token},\\
1, & \text{if } i=j,\\
0, & \text{otherwise}.
\end{cases}
\end{equation}
Context tokens attend to other context tokens, allowing the model to infer task-specific structure from the observed dataset. Query tokens attend to the context set and to their own covariate representation, but not to the responses of other query tokens. This prevents leakage across query observations while allowing each forecast to condition on the full censored history. 

LT-ICL uses a conditional normalizing-flow head to parameterize a flexible predictive density over lead times. The flow maps a simple base random variable $Z\sim p_0$ to a lead time variable $Y$ through a series of invertible transformations
\begin{equation}
Y = f_\phi(Z;x, D) = \left (f_\phi^{(K)} \circ \cdots \circ f_\phi^{(1)} \right)(Z;x, D).
\end{equation}
The corresponding density is evaluated using the change-of-variables formula,
\begin{equation}
q_{\phi}(Y\mid x, D) = p_0(f_{\phi}^{-1}(Y;x, D))
\left| \frac{\partial f_{\phi}^{-1}(Y;x, D)}{\partial Y} \right|.
\end{equation}
In our implementation, the flow uses a series of affine coupling and rational-quadratic spline transformations, allowing the predictive distribution to capture skewness and heavy tails commonly observed in supplier lead times. Since lead times are positive, the flow is parameterized on a transformed positive support and mapped back to the original lead time scale for evaluation and forecasting. \add{For an open query order, the predictive distribution is conditioned on its elapsed duration, $q_\phi(y \mid y > C-d, x, D)$.} \autoref{fig:transformer-flow-head} presents an overview of the proposed LT-ICL architecture. 

\begin{figure*}[t]
    \centering
    \includegraphics[width=0.85\linewidth]{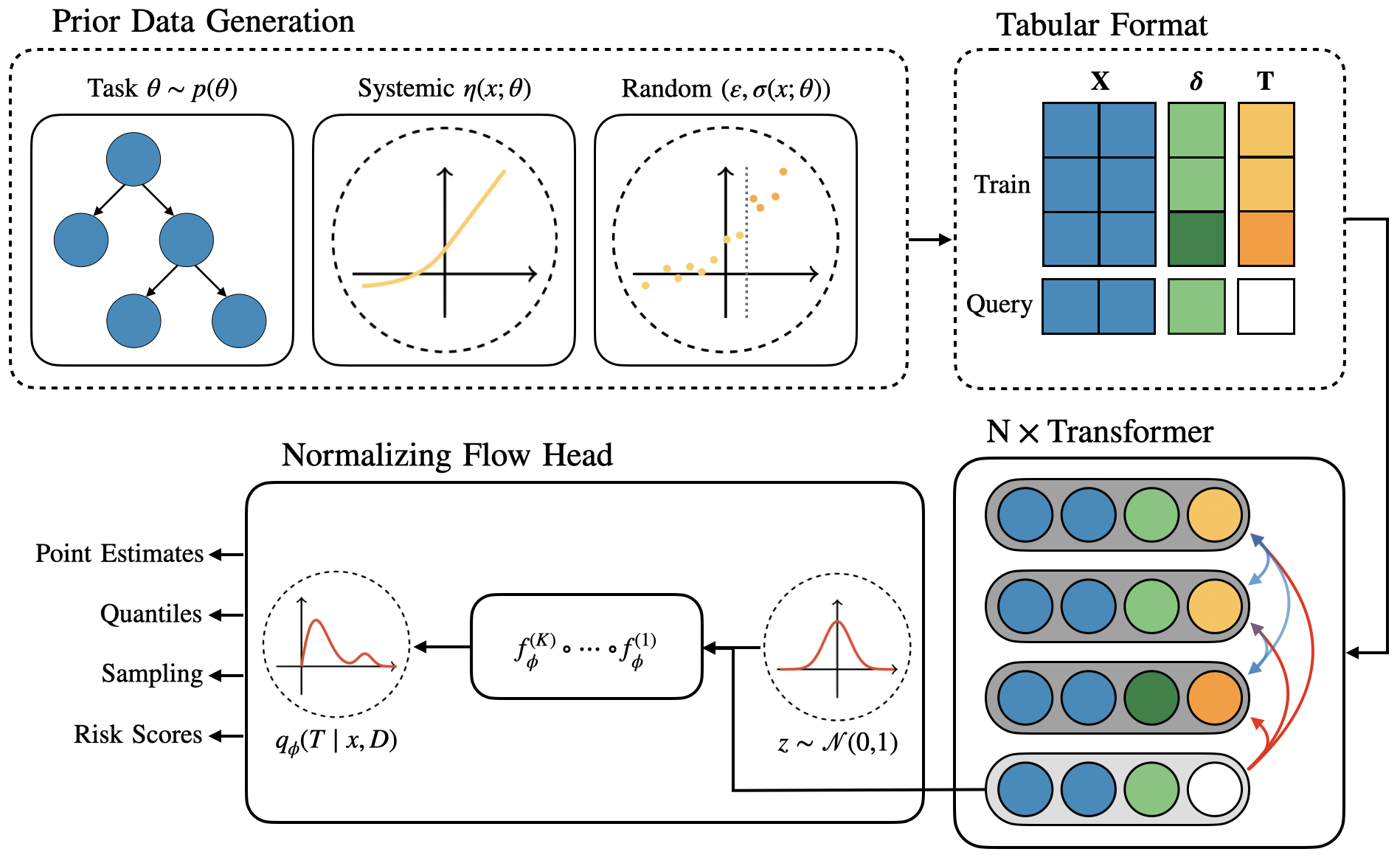}
    \caption{Architecture of LT-ICL for censoring-aware lead time forecasting. Arrows illustrate attention-based information flow between training-context and query tokens.}
    \label{fig:transformer-flow-head}
\end{figure*}

\subsection{Training Objective}
LT-ICL is trained on synthetic censored lead time tasks sampled from the prior described in Section~\ref{subsec:synthetic_prior}. Each synthetic task $\theta_i \sim p(\theta)$ defines a data-generating process from which we sample a context set $D_i \sim p(x, Y \mid \theta_i)$ and a query set $\{x_j\}_{j=1}^{N}\sim p(x \mid \theta_i)$. The context set contains completed and right-censored orders observed at the forecasting cutoff. The query set contains orders for which the model must produce predictive lead time distributions. 

Although the model does not observe $\theta_i$, minimizing the expected negative log likelihood over tasks, contexts, and query responses trains $q_\phi$ to approximate the prior-induced posterior predictive distribution $p(Y\mid x,D)$ \cite{nagler2023statisticalfoundationspriordatafitted}. For each synthetic task $i$, query order $j$, and Monte Carlo response sample $k$, let $Y_{ijk} \sim p(Y \mid x_{ij}, \theta_i)$ denote a sampled lead time from the corresponding target predictive distribution. The distributional training loss is then
\begin{equation}
\mathcal{L}_{\mathrm{NLL}}(\phi) = \frac{1}{M}\sum_{i=1}^{M}\frac{1}{N}\sum_{j=1}^{N}\frac{1}{O}\sum_{k=1}^{O}-\log q_\phi(Y_{ijk}\mid x_{ij},D_{i}).
\end{equation}
where $M$ is the number of sampled tasks, $N$ is the number of query orders per task, and $O$ is the number of response samples per query order. 

When $O=1$, this reduces to the prior-data negative log-likelihood objective used in PFN training \cite{müller2024transformersbayesianinference}. In expectation under the synthetic task prior, this objective minimizes the KL divergence between the prior-induced posterior predictive distribution and the amortized predictive distribution \cite{nagler2023statisticalfoundationspriordatafitted}. When $O>1$, the objective averages over multiple plausible lead times for the same query order, giving a lower-variance Monte Carlo estimate of the cross entropy between the synthetic posterior predictive distribution and the model prediction. 

Although likelihood training targets the full predictive distribution, generally point estimates are required by MRP production systems. We therefore add an auxiliary forecasting loss given by the $L_{1}$ distance on the predictive median. Let $\hat m_\phi(x,D) = F_{q_\phi}^{-1}(0.5 \mid x,D)$ denote the median implied by the normalizing flow head, and let $m_\theta(x)$ denote the true median lead time under synthetic task $p(Y\mid x; \theta_i)$. The forecasting loss is 
\begin{equation}
\mathcal{L}_{\mathrm{Point}}(\phi) = \frac{1}{M}
\sum_{i=1}^{M}\frac{1}{N}\sum_{j=1}^{N}\left \vert\hat m_\phi(x_{ij},D_i) - m_{\theta}(x_{ij}) \right\vert.
\end{equation}
The final objective combines the distributional and forecasting terms:
\begin{equation}
\mathcal{L}(\phi) = \mathcal{L}_{\mathrm{NLL}}(\phi) + \mathcal{L}_{\mathrm{Point}}(\phi).
\end{equation}

\subsection{Prior Data Generation}\label{subsec:synthetic_prior}
Our desiderata for a prior data generation pipeline are that: it mimics characteristics of real supplier lead time data, it is diverse so as to emulate a variety of industries, and it is computationally fast to execute. We decompose the generative process into three components: a \emph{task structure} $\theta$ that defines the structural skeleton of the dataset, a \emph{systematic component} $\eta(x;\theta)$ that deterministically maps covariates to a location parameter, and a \emph{random component} $(\varepsilon, \sigma(x;\theta))$ that introduces stochastic variation. This decomposition enables our prior to flexibly model a variety of conditions. Moreover, it lets the prior compose simple building blocks into a rich range of prediction tasks and enables finer control over prior complexity. 

\subsubsection{Task Structure}
A task $\theta$ can be viewed as a simulated supply chain environment that defines the parts, sites, and suppliers involved, their frequencies of occurrence, the observation window, and the nature of the underlying stochasticity. To mimic the sparsity of real order histories, valid part--site--supplier triples are generated from Zipf-weighted marginal distributions rather than sampled uniformly. This creates long-tailed categorical frequencies in which a small number of groups appear frequently while many combinations are rare. A cutoff date $C$ is then sampled inside the observation window, inducing the same censoring mechanism found in operational settings. Finally, a random family $\varepsilon \in \{\mathrm{AFT}, \mathrm{PH}\}$ is sampled to determine whether the task follows an accelerated failure time or proportional hazards model. Algorithm \ref{alg:task} provides an overview of the task generation process. 

\input{algorithms/task}

\subsubsection{Systematic Component}
The systematic component $\eta(x; \theta)$ deterministically maps the covariates of an observation $x_{i}$ to a scalar which serves as the location parameter. Let $\mathcal{J} = \{\text{part, site, supplier, group}\}$ be the set of hierarchy levels, $g_{j}(x)$ be the group identity of $x$ at level $j \in \mathcal{J}$, $x^{\mathrm{order}} \in \mathbb{R}^d$ be the order-level covariates, and $x^{\mathrm{time}}$ be time-based covariates. Then the systematic component is given by 
\begin{dmath}
\eta(x; \theta) = \mu_{0} + \sum_{j \in \mathcal{J}} \left(\boldsymbol{\beta}_{g_{j}(x)}^{\top}\phi_{g_{j}(x)}\left(x^{\mathrm{order}}\right) 
\quad +\sum_{m\in g_{j}(i)}\alpha_{g_{j}(x)}^{(m)}K_{g_{j}(x)}\left(x^{\mathrm{time}}, x^{\mathrm{time}}_{m}\right)\right) ,
\end{dmath}
where $\mu_{0}$ is a global offset, $\boldsymbol{\beta}_{g_{j}(x)}$ and $\boldsymbol{\alpha}_{g_{j}(x)}$ are group-specific coefficients, and $\phi_{g_{j}(x)}(\cdot)$ and $K_{g_{j}(x)}(\cdot, \cdot)$ are the group-specific basis and kernel transformations. 
 
Therefore, $\eta(x;\theta)$ can model both parametric and nonparametric effects, providing a rich class of learnable mechanisms. Furthermore, the latent group effects introduce correlated behavior and \rev{allows }{allow} covariate effects to vary in different procurement contexts. For example, an increase in order quantity may have little effect for one supplier but substantially increase lead times for another. Algorithm \ref{alg:systematic} provides an overview of the systematic data generation process. 

\input{algorithms/systemic}

\subsubsection{Random Component and Observation Process}
The random component consists of a stochastic family $\varepsilon$ and a covariate-dependent scale $\sigma(x;\theta)$. The first term, $\varepsilon \in \{AFT, PH\}$, specifies the form of stochasticity in the model. The second term $\sigma(x;\theta)$ introduces group-level heteroskedasticity that is inferrable from context. 

In an accelerated failure time (AFT) specification, randomness enters additively in the logarithm of the lead time. In contrast, under a proportional hazards (PH) specification, randomness is introduced through the hazard function, where $\eta(x; \theta)$ acts multiplicatively on the hazard rate. The survival function of our prior is given by 
\begin{equation}
S(T \mid x; \theta) =
\begin{cases}
S_{\mathrm{AFT}}\left(\frac{\log T - \eta(x; \theta)}{\sigma(x;\theta)}\right), & \text{if } \varepsilon = \mathrm{AFT}, \\
S_{\mathrm{PH}}(T; \sigma(x; \theta))^{\exp(\eta(x; \theta))}, & \text{if } \varepsilon = \mathrm{PH}.
\end{cases}
\end{equation}
where $S_{\mathrm{AFT}}(T \mid x; \theta)$ and $S_{\mathrm{PH}}(T \mid x; \theta)$ are the survival functions defined by $f_\varepsilon$. We use inverse transform sampling with the survival function to generate samples. For an overview of the random component generation, see Algorithm \ref{alg:random}. 

\input{algorithms/random}

\subsection{CRPS Risk Decomposition}\label{subsec:risk_decomposition}
The prior generator and amortized training objective play distinct roles in LT-ICL. The prior determines the family of lead time processes over which the model is trained, while the neural model determines how accurately the corresponding posterior predictive distribution can be approximated. We derive an excess forecasting risk decomposition based on the probabilistic forecasting metric Continuous Ranked Probability Score (CRPS) \cite{waghmare2026properscoringrulesestimation}. We show that \rev{LT-ICL}{PFN-based ICL} forecasting error is controlled by two terms: prior misspecification, determined by the realism of the synthetic task prior, and amortized approximation, determined by the model’s ability to approximate the prior-induced posterior predictive distribution.

Let $p^*(Y\mid x,D)$ denote the true predictive distribution of a lead time $Y$ for query covariates $x$ given context set $D$. Let
\begin{equation}
\pi(Y\mid x,D) = \int p(Y\mid x,\theta)p(\theta\mid D)d\theta
\end{equation}
denote the posterior predictive distribution induced by the synthetic task prior $p(\theta)$, and let $q_\phi(Y\mid x,D)$ denote the amortized predictive distribution learned by LT-ICL. The excess \rev{continous ranked probablisty }{continuous ranked probability} score risk is defined as
\begin{equation}
\mathcal{E} = \mathbb{E}_{y\sim p^{*}(Y\mid x, D)}\left[\mathrm{CRPS}(q_{\phi},Y) - \mathrm{CRPS}(p^{*},Y)\right]
\end{equation}
where the expectation is taken under the true data-generating distribution.

\begin{proposition}[Forecast Risk Decomposition] \label{prop:decomp}
Assume lead times are bounded, $Y\in[0,\bar{Y}]$, and the relevant KL divergences
are finite. Then
\begin{equation}
\mathcal{E} \leq \bar{Y}
\left(
\underbrace{\mathrm{KL}(p^{*}\Vert \pi)}_{\mathrm{misspecification}}
+
\underbrace{\mathrm{KL}(\pi\Vert q_{\phi})}_{\mathrm{approximation}}
\right).
\end{equation}
\end{proposition}

We present the proof for Proposition \ref{prop:decomp} below:
\begin{proof}
\begin{align*}
\mathcal{E} &=  \int_{0}^{\bar{Y}} (F_{q}(t) - F_{p}(t))^{2}dt  \\
&=  \int_{0}^{\bar{Y}} (F_{q}(t) -  F_{\pi}(t) + F_{\pi}(t) - F_{p}(t))^{2}dt  \\
&\leq  \int_{0}^{\bar{Y}} 2(F_{q}(t) -  F_{\pi}(t))^{2}dt +  \int_{0}^{\bar{Y}}2(F_{\pi}(t) - F_{p}(t))^{2}dt \\
&\leq  \int_{0}^{\bar{Y}} 2(\mathrm{TV}(q_{\phi}, \pi))^{2}dt +  \int_{0}^{\bar{Y}}2(\mathrm{TV}(\pi, p^{*}))^{2}dt \\
&=  2\bar{Y}(\mathrm{TV}(q_{\phi}, \pi))^{2} + 2\bar{Y}(\mathrm{TV}(\pi, p^{*}))^{2} \\
&\leq \bar{Y} \left(\mathrm{KL}( \pi \Vert q_{\phi}) + \mathrm{KL}( p^{*} \Vert \pi) \right)  
\end{align*}
\end{proof}
where the first line follows from the divergence induced by CRPS \cite{waghmare2026properscoringrulesestimation}, the fourth line follows from the definition of total variation distance and the final line follows from Pinsker's inequality. \add{The boundedness assumption is reasonable in practice, as orders that exceed a certain lead time threshold are often considered incomplete or lost. LT-ICL imposes a maximum predictive lead time of 3650 days, or around 10 years. 
}

The first term is a prior misspecification error. It measures the discrepancy between the true lead time predictive distribution and the posterior predictive distribution induced by the synthetic prior. Reducing this term requires the prior generator to cover the structural features of industrial lead time data. The second term is an amortized approximation error. It measures how closely \rev{LT-ICL}{the ICL method} approximates the posterior predictive distribution induced by the synthetic prior. This is the term targeted during PFN pretraining: under the synthetic task distribution, the expected negative log-likelihood is minimized when $q_\phi(Y\mid x,D)=\pi(Y\mid x,D)$. This decomposition identifies two levers for \add{improving ICL methods such as LT-ICL}. The synthetic prior controls the misspecification term, while amortized training controls the approximation term. Thus, \rev{LT-ICL}{an ICL method's} forecasting performance depends jointly on the realism of the prior data generator and the accuracy of amortized posterior predictive inference.

%% file: algorithms/task.tex
\begin{algorithm}[t]
\caption{Task Structure $\theta$}\label{alg:task}
\KwOut{Task parameters $\theta = \bigl(\mathcal{G},\, \{x_i\}_{i=1}^N,\, C,\, \varepsilon,\, f_\varepsilon\bigr)$}
\tcp{Sample group cardinalities}
$n_{\mathrm{part}},\, n_{\mathrm{site}},\, n_{\mathrm{supplier}}, \, n_{\mathrm{groups}} \sim \mathrm{Uniform}(\mathbb{Z})$\; 

\tcp{Assign Zipf-weighted marginals}
\ForEach{$j \in \{\mathrm{part},\, \mathrm{site},\, \mathrm{supplier}\}$}{
    $\alpha_j \sim \mathrm{Beta}(0.5,\, 0.5)$\; \\
    $p_j(k) \propto \mathrm{rank}_k^{-\alpha_j}$\;
}
\tcp{Build sparse set of valid triplets}
$\mathcal{G} \gets \emptyset$\;\\
\While{$|\mathcal{G}| < n_{\mathrm{groups}}$}{
    $(r_{\mathrm{part}},\, r_{\mathrm{site}},\, r_{\mathrm{supplier}}) \sim p_{\mathrm{part}} \times p_{\mathrm{site}} \times p_{\mathrm{supplier}}$\; \\
    $\mathcal{G} \gets \mathcal{G} \cup \{(r_{\mathrm{part}},\, r_{\mathrm{site}},\, r_{\mathrm{supplier}})\}$\;
}
\tcp{Sample temporal window and cutoff}
$t_0,\, \Delta \sim \mathrm{Uniform}(\mathbb{R})$\;
$C \sim \mathrm{Uniform}(t_0,\, t_0 + \Delta)$\; 

\tcp{Sample error distribution}
$\varepsilon \sim \mathrm{Uniform}\{\mathrm{AFT},\, \mathrm{PH}\}$\;

\eIf{$\varepsilon = \mathrm{AFT}$}{
    $f_\varepsilon \sim \mathrm{ErrorBank}(\text{Gaussian, logistic, Gumbel,}  \dots)$\;
}{
    $f_\varepsilon \sim \mathrm{HazardBank}(\text{Weibull, log-logistic, Gompertz,} \dots)$\;
}
\tcp{Construct observation features}
\ForEach{observation $i = 1,\, \dotsc,\, N$}{
    $x_i \sim \mathrm{Categorical}(\mathcal{G})\times \mathrm{Uniform}(t_0,\, t_0 + \Delta)$\;
}
\Return{$\theta = \bigl(\mathcal{G},\, \{x_i\}_{i=1}^N,\, C,\, \varepsilon,\, f_\varepsilon\bigr)$}
\end{algorithm}

%% file: algorithms/systemic.tex
\begin{algorithm}[t]
\caption{Systematic Component $\eta(x;\,\theta)$}\label{alg:systematic}
\KwIn{Task structure $\theta$}
\KwOut{Linear predictors $\{\eta_i\}_{i=1}^N$}
\tcp{Sample global offset}
$\mu_0 \sim \mathrm{Uniform}(\mathbb{R})$\;

\tcp{Sample group-specific transforms}
\ForEach{level $j \in \mathcal{J}$ and identity $g \in \mathcal{G}_j$}{
    $\phi_g \sim \mathrm{TransformBank}(\text{sigmoid, step, ReLU, RBF, \dots})$\;
    
    $K_g \sim \mathrm{KernelBank}(\text{periodic, RBF, Mat\'{e}rn, \dots})$\;
    
    $\boldsymbol{\beta}_g  \sim \mathcal{N}(\mu_g,\, \sigma_g)$\;

}
\tcp{Compute linear predictor}
\ForEach{observation $i = 1,\, \dotsc,\, N$}{
    $\boldsymbol{\alpha}_{g_j(x_i)} \sim \mathcal{N}(\mathbf{0},\, K_{g_j(x_i)}^{-1}(x^{\mathrm{time}}_i; g_j))$\;
    
    $\displaystyle \eta_i \gets \mu_0 + \sum_{j \in \mathcal{J}} \left( \boldsymbol{\beta}_{g_j(x_i)}^\top \phi_{g_j(x_i)}(x_i^{\mathrm{order}}) \right)$\;

    $\displaystyle \eta_i \gets \eta_i + \sum_{j \in \mathcal{J}} \left( \boldsymbol{\alpha}_{g_j(x_i)}^\top K_{g_j(x_i)}(x_i; g_j) \right)$\;
}
\Return{$\{\eta_i\}_{i=1}^N$}
\end{algorithm}

%% file: algorithms/random.tex
\begin{algorithm}[t]
\caption{Random Component and Observation Process}\label{alg:random}
\KwIn{Task structure $\theta$, linear predictors $\{\eta_i\}_{i=1}^N$}
\KwOut{Context set $D$, query set $Q$}
\tcp{Sample group-specific scales}
\ForEach{level $j \in \mathcal{J}$ and identity $g \in \mathcal{G}_j$}{
    $\sigma_g \sim \mathrm{Uniform}(\mathbb{R}^+)$\;
}
\tcp{Sample lead times}
\ForEach{observation $i = 1,\, \dotsc,\, N$}{
    $\displaystyle \sigma_i \gets \sum_{j \in \mathcal{J}} \sigma_{g_j(x_i)}$\;
    
    $U_i \sim \mathrm{Uniform}(0,\,1)$\;
    
    $Y_i \gets S^{-1}(U_i \mid \eta_i,\, \sigma_i,\, \theta)$\;
}
\tcp{Apply censoring and split}
$D \gets \emptyset, \quad Q \gets \emptyset$\;

\ForEach{observation $i$ with $d_i < C$}{
    \eIf{$d_i + Y_i \leq C$}{
        $D \gets D \cup \{(x_i,\, Y_i,\, \delta_i = 1)\}$\;
    }{
        $T_i \gets C - d_i $
        
        $D \gets D \cup \{(x_i,\, T_i,\, \delta_i = 0)\}$\;
        
        $Y_i' \gets S^{-1}(U_i' \mid Y_i > C - d_i,\, \eta_i,\, \sigma_i,\, \theta)$\;
        
        $m_i \gets S^{-1}(0.5 \mid \eta_i,\, \sigma_i,\, \theta)$\;
        
        $Q \gets Q \cup \{(x_i,\, Y_i',\, m_i)\}$\;
    }
}
\tcp{Future orders handled analogously}
\Return{$(D,\, Q)$ } 
\end{algorithm}

%% file: _04_case_study.tex
\section{Case Study}\label{sec:case_study}
\subsection{Dataset Details}
Historical supplier lead time datasets were extracted from an advanced sales and operations planning platform. These datasets span seven industries (Pharmaceutical (7), Electronics (5), Manufacturing (4), Medical Devices (3), Chemical (3),  Semiconductor (1), and Automotive (1)), resulting in a diverse benchmark. Each dataset consists of real procurement records linking ordered parts to supplier and site information, together with the order and receipt dates required to measure realized lead times. Orders that remained incomplete at the time of data extraction were excluded and right-censoring was introduced retrospectively through the experimental design to evaluate forecasting performance under partially observed lead time histories. Customer identities are anonymized in accordance with existing agreements.

\subsection{Data Preprocessing}
Part-site-supplier combinations with fewer than 25 observations were removed, in a manner similar to prior case studies \cite{Brintrup02062020, STEINBERG2023100003}. Only variables available for all participating customers were retained, namely $\texttt{part id}$, $\texttt{site id}$, $\texttt{supplier id}$, $\texttt{order id}$, $\texttt{order date}$, $\texttt{receipt date}$, and $\texttt{order quantity}$. This restriction improves comparability across datasets and removes customer-specific attributes that could reveal the identity of participating firms. Duplicate records and records with missing values in the harmonized schema were excluded from analysis following \cite{Brintrup02062020, STEINBERG2023100003}. \rev{Finally, to preserve confidentiality, part, site, and supplier identifiers were anonymized, and each dataset was subsampled to similar sizes to further protect privacy.}{Finally, to preserve confidentiality, identifiers were anonymized and datasets subsampled to comparable sizes. Both steps were performed before the temporal split, independently of lead times and order dates, thereby avoiding target leakage through these operations.} \autoref{tab:dataset_summary} reports summary statistics for the anonymized datasets\add{, including the proportion of incomplete orders removed prior to subsampling}. 

\begin{table}[t]
\raggedright
\caption{Dataset-level summary statistics.}
\label{tab:dataset_summary}
\footnotesize
\setlength{\tabcolsep}{3pt}
\input{tables/summary}
\end{table}

\subsection{Benchmark Methods}
We evaluate LT-ICL against popular machine learning, deep learning, and in-context learning methods. All baselines were presented with the same feature transformations, with categorical variables encoded using one-hot encodings, calendar dates mapped to harmonic Fourier features to capture periodic structure, a running index to capture trend, and order quantities $q_i$ standardized ($\tilde{q}_i = \frac{q_i - \mu_q}{\sigma_q}$) to obtain a unit-free measure comparable across datasets. Whenever possible, we compare LT-ICL against both regression and survival variants of comparative models.\\

\noindent \textbf{Tree Ensembles.}
We include random forests (RF) and extra trees (ET), which are commonly used in lead time estimation and related supply-chain prediction tasks \cite{BARROS2023106671,pom,CAMUR2024123226}. We also include their survival counterparts, random survival forests (RSF) \cite{Ishwaran_2008} and extra survival trees (EST), which construct ensembles from censoring-aware survival trees \cite{Leblanc01061993}.\\

\noindent \textbf{Boosted Trees.}
We include XGBoost \cite{Chen_2016}, which has been applied to delivery-date prediction in manufacturing and construction settings \cite{11362374,s10845-023-02290-2}. In addition to standard XGBoost regression (XGB-Reg), we evaluate XGBoost with accelerated-failure-time and proportional-hazards objectives, denoted XGB-AFT and XGB-PH \cite{Barnwal02102022}. These variants allow the same boosted-tree family to serve as either a point forecaster or a censoring-aware probabilistic forecaster.\\

\noindent \textbf{Tabular Neural Models.}
We include TabM  \cite{gorishniy2025tabmadvancingtabulardeep}, a parameter efficient ensemble of multilayer perceptrons as a stand-in for classic MLPs \cite{Schneckenreither03062021} (TabM-Reg). To obtain deep censoring-aware baselines, we also train TabM-based accelerated-failure-time (TabM-AFT) and proportional-hazards models (TabM-PH), which use the DeepAFT \cite{norman2024deepaft} and DeepSurv \cite{Katzman_2018} prediction heads, respectively.\\

\noindent \textbf{Recurrent Neural Models.}
We include Long Short-Term Memory networks as a commonly used neural baseline (LSTM-Reg) \cite{gabellini,11468209}, along with survival variants based on accelerated failure time and proportional hazards objectives, denoted LSTM-AFT and LSTM-PH, respectively. For orders that occur at the same time step, we impose a random ordering to construct input sequences.\\

\noindent \textbf{Transformer-Based Models.}
We evaluate task-specific Transformer architectures \cite{vaswani2023attentionneed}. Prior work in this setting can be interpreted as an accelerated failure time model with a log-normal likelihood \cite{MATTHES2025807}, which we denote Transformer-AFT; however, this model does not incorporate censoring information. We therefore complement it with a regression Transformer (Transformer-Reg) and a proportional hazards Transformer (Transformer-PH), incorporating censoring information wherever the objective allows.\\

\noindent \textbf{Tabular Foundation Methods.}
We compare against contemporary in-context tabular foundation models: TabPFNv2 \cite{Hollmann2025}, TabICLv2 \cite{qu2026tabiclv2betterfasterscalable}, and Mitra \cite{zhang2025mitramixedsyntheticpriors}. These methods are transformer models pretrained on synthetic tabular tasks and adapt to new datasets through in-context prediction, but they are not designed to directly exploit right-censored lead time observations. \\

\noindent \textbf{Survival In-Context Learning Methods.} \add{We evaluate two categories of survival-specific zero-shot foundation methods: methods that adapt existing tabular foundation models and methods pretrained directly on synthetic survival tasks. For the former, we include TabPFN-CF  \cite{kim2026tabularfoundationmodelssurvival} and TabSA \cite{vargasvieyra2026staying}. We exclude TabFM-MTLR \cite{10.1007/978-3-032-35387-0_23} because it requires fitting multiple additional parameters. For the latter, we include SurvivalPFN  \cite{qi2026survivalpfn} and Survival In-Context (SIC) \cite{seletkov2026survivalincontextamortizedbayesian}. We exclude SurvPFN \cite{wang2026survpfn}  because its implementation supports fewer (up to 10) covariates  than is required by our evaluation.}

\subsection{Performance Metrics}
All models were evaluated on point-prediction performance, reflecting the lead time inputs commonly required by MRP systems. Point-prediction accuracy was assessed using weighted mean absolute percentage error (WMAPE), a widely used scale-free forecasting metric for lead time estimation \cite{BARROS2023106671, s10845-023-02290-2}:
\begin{equation}
\mathrm{WMAPE} = \frac{\sum_{i=1}^{N} \left \vert y_i - \hat{y}_i \right \vert}{\sum_{i=1}^{N} \left \vert  y_i \right \vert}.
\end{equation}
WMAPE measures aggregate absolute deviation relative to total observed lead time, yielding an interpretable percentage error that is comparable across datasets with different lead time scales. \add{We evaluate the predictive median as the point estimate} and report WMAPE for the overall test set, as well as separately for open and future orders. 

For methods that estimate a predictive distribution, we evaluate probabilistic accuracy using scaled continuous ranked probability score (SCRPS), a probabilistic forecasting metric \cite{potosnak2026forkingsequences}. Writing $F_i$ for the predictive cumulative distribution function for instance $i$, SCRPS is defined as
\begin{equation}
\mathrm{SCRPS} =
\frac{\sum_{i=1}^{N}
\int_{0}^{\infty} \left(F_i(z) - \mathbb{I}\{z \geq y_i\}\right)^2  dz
}{
\sum_{i=1}^{N} \vert y_i\vert
}.
\end{equation}
The numerator is the aggregate CRPS over the test set, while the denominator normalizes by total observed lead time to obtain a scale-free metric. This metric is related to Proposition~\ref{prop:decomp}, which analyzes excess CRPS risk, while also enabling comparison across datasets with different lead time scales. When the predictive distribution is a point mass at $\hat y_i$, the CRPS for instance $i$ reduces to $\vert y_i-\hat y_i\vert$; therefore, SCRPS reduces to WMAPE for deterministic point forecasts. SCRPS  is reported for the overall test set, as well as separately for open and future orders.

\add{
To assess upper-tail predictive accuracy, we evaluate scaled pinball loss (SPL) at quantile levels $\tau \in \{0.80, 0.90, 0.95\}$. Pinball loss asymmetrically penalizes prediction errors, reflecting the potentially greater operational consequences of lead time underestimation over overestimation \cite{ZHANG2026109985}. Letting $\hat{q}_{i,\tau}=F_i^{-1}(\tau)$ denote the predicted $\tau$-quantile for instance $i$, we define
\begin{equation}
\mathrm{SPL}_{\tau} =
\frac{
\sum_{i=1}^{N}
\max\left\{
\tau(y_i-\hat{q}_{i,\tau}),
(\tau-1)(y_i-\hat{q}_{i,\tau})
\right\}
}{
\sum_{i=1}^{N}|y_i|
}.
\end{equation}
As with SCRPS, normalization by total observed lead time enables comparison across datasets with different lead time scales.
}

Because the evaluation spans 24 datasets, \rev{20 models}{24 methods}, and 3 evaluation settings, we summarize performance using aggregate statistics across datasets rather than reporting dataset-level results. For each method, we report the average error metric, either WMAPE or SCRPS, the average rank relative to competing methods, the total number of wins (the number of datasets on which the method achieves the lowest error), and the average pairwise win rate (the proportion of model--dataset comparisons in which the method outperforms a competing method). \add{For SPL, we report the mean and standard deviation of the dataset-level scores across all 24 datasets for each quantile level.} Together, these statistics capture both overall accuracy and consistency of performance across heterogeneous lead time forecasting settings.

\subsection{Implementation Details}

\begin{table}[t]
\raggedright
\caption{\add{Fold-level Summary Statistics}}
\label{tab:fold_summary}
\footnotesize
\input{tables/fold}
\end{table}

\rev{For evaluation, each dataset was partitioned into disjoint validation folds. A forecast cutoff date was selected and used to temporally split observations into train and test sets for each fold. Orders placed before the chosen cutoff but not yet received by the cutoff are treated as right-censored.}{For evaluation, each dataset was partitioned into 10 disjoint folds of approximately 1000 orders. For each fold, the forecast cutoff was selected as the earliest date leaving at most 224 query orders, to stay within LT-ICL's input limit of 1024 context and 224 query observations. The context set contains all orders placed before the cutoff; open orders enter it as right-censored observations and therefore appear in both sets. Across all 240 folds, query sets contained between 192 and 224 orders and context sets between 777 and 992. \autoref{tab:fold_summary} provides fold-level statistics. All methods were evaluated on the same folds and cutoffs.}

Each benchmark model's hyperparameters were tuned independently per fold per dataset using grid search (an overview of hyperparameters can be found in \autoref{tab:hyperparameters}). \rev{Total runtime time for supervised learning models is then measured as the combined runtime of hyperparameter tuning, training and inference. In contrast, the runtime for in-context learning methods is measured only as inference time, which performs the same function as the previous three stages }{We measure deployment-time runtime as the computation required to adapt a method to a new dataset. For supervised baselines, this includes hyperparameter tuning, model training, and inference. For in-context learning methods, adaptation occurs through the context, so deployment primarily consists of preprocessing and inference}. WMAPE, SCRPS, \add{and SPL} calculations are averaged across folds per dataset to establish a stable performance estimate.  

\begin{table}[ht]
\centering
\caption{Hyperparameter search grids by model family. Best configurations are selected per fold by grid search on a validation split.}
\input{tables/hyperparameters}

\end{table}

To ensure a fair runtime comparison, all benchmark models were trained and evaluated on the same machine equipped with an Intel(R) Xeon(R) Platinum 8481C CPU at 2.70 GHz and four NVIDIA H100 GPUs. GPU acceleration was used whenever applicable. The LT-ICL model was pretrained once for approximately nine hours using data parallelism across the four GPUs. Training consisted of 35,000 gradient steps with a per-GPU batch size of 512, corresponding to more than 70 million synthetic lead time datasets. Because of the computational cost of pretraining, no hyperparameter tuning was performed for LT-ICL. The final model contains 26 million parameters and was trained using the Schedule-Free AdamW optimizer with learning rate $\eta = 10^{-5}$ \cite{defazio2024roadscheduled}.

%% file: tables/summary.tex
\begingroup
\setlength{\tabcolsep}{2pt}
\renewcommand{\arraystretch}{0.85}
\begin{tabular}{@{}llrcrc@{}}
\toprule
\textbf{ID} & \textbf{Industry} & $\mathbf{Length}$ & \textbf{LT (days)} & \textbf{Incomplete (\%)} & \textbf{Date range} \\
\midrule
A & Automotive & 10046 & 111.0$\,\pm\,$83.0 & 0.3 & 2022--26 \\
B & Chemical & 9977 & 66.2$\,\pm\,$54.8 & 0.0 & 2012--18 \\
C & Chemical & 10009 & 85.7$\,\pm\,$72.2 & 0.4 & 2023--26 \\
D & Chemical & 10002 & 44.2$\,\pm\,$37.8 & 0.2 & 2018--26 \\
E & Electronics & 9994 & 89.4$\,\pm\,$45.8 & 0.7 & 2023--26 \\
F & Electronics & 9997 & 123.2$\,\pm\,$89.0 & 0.0 & 2016--19 \\
G & Electronics & 9995 & 54.9$\,\pm\,$65.7 & 0.0 & 2022--25 \\
H & Electronics & 10053 & 93.0$\,\pm\,$87.5 & 0.0 & 2012--17 \\
I & Electronics & 9975 & 112.9$\,\pm\,$52.3 & 0.6 & 2022--26 \\
J & Manufacturing & 9977 & 13.0$\,\pm\,$13.0 & 0.0 & 2022--25 \\
K & Manufacturing & 10009 & 12.7$\,\pm\,$28.6 & 0.1 & 2022--26 \\
L & Manufacturing & 9987 & 6.3$\,\pm\,$5.0 & 0.0 & 2024--26 \\
M & Manufacturing & 10020 & 120.1$\,\pm\,$69.7 & 0.0 & 2021--25 \\
N & Medical Devices & 9992 & 84.4$\,\pm\,$101.0 & 0.1 & 2019--26 \\
O & Medical Devices & 9981 & 35.3$\,\pm\,$38.5 & 0.3 & 2020--26 \\
P & Medical Devices & 10005 & 160.6$\,\pm\,$112.1 & 0.0 & 2021--26 \\
Q & Pharmaceutical & 10003 & 68.9$\,\pm\,$76.8 & 0.9 & 2021--26 \\
R & Pharmaceutical & 9992 & 32.7$\,\pm\,$79.6 & 3.0 & 2023--26 \\
S & Pharmaceutical & 9976 & 105.2$\,\pm\,$91.5 & 0.0 & 2021--25 \\
T & Pharmaceutical & 10029 & 11.7$\,\pm\,$14.3 & 0.0 & 2022--25 \\
U & Pharmaceutical & 10049 & 41.5$\,\pm\,$69.7 & 0.4 & 2018--26 \\
V & Pharmaceutical & 10063 & 25.1$\,\pm\,$46.1 & 0.4 & 2009--26 \\
W & Pharmaceutical & 10032 & 51.3$\,\pm\,$52.5 & 2.8 & 2016--26 \\
X & Semiconductor & 10048 & 45.3$\,\pm\,$92.8 & 0.0 & 2022--26 \\
\bottomrule
\end{tabular}
\endgroup

%% file: tables/fold.tex
\begingroup
\setlength{\tabcolsep}{4pt}
\renewcommand{\arraystretch}{0.85}
\begin{tabular}{@{}lccccc@{}}
\toprule
\multirow{2}{*}{\textbf{ID}} &
\multirow{2}{*}{\textbf{Context}} &
\multirow{2}{*}{\textbf{Query}} &
\multirow{2}{*}{\textbf{Open Orders}} &
\textbf{Open Orders} &
\textbf{Open Orders} \\
& & & & \textbf{(\% Context)} & \textbf{(\% Query)} \\
\midrule
A & 869$\,\pm\,$56 & 223$\,\pm\,$1 & 88$\,\pm\,$55 & 9.7$\,\pm\,$5.6 & 39.4$\,\pm\,$25.0 \\
B & 812$\,\pm\,$10 & 223$\,\pm\,$1 & 37$\,\pm\,$10 & 4.5$\,\pm\,$1.2 & 16.5$\,\pm\,$4.6 \\
C & 879$\,\pm\,$21 & 220$\,\pm\,$6 & 98$\,\pm\,$22 & 11.1$\,\pm\,$2.2 & 44.3$\,\pm\,$9.7 \\
D & 794$\,\pm\,$10 & 222$\,\pm\,$1 & 16$\,\pm\,$10 & 2.0$\,\pm\,$1.2 & 7.3$\,\pm\,$4.5 \\
E & 888$\,\pm\,$28 & 221$\,\pm\,$1 & 110$\,\pm\,$28 & 12.3$\,\pm\,$2.8 & 49.8$\,\pm\,$12.6 \\
F & 919$\,\pm\,$35 & 222$\,\pm\,$2 & 142$\,\pm\,$34 & 15.3$\,\pm\,$3.1 & 63.8$\,\pm\,$15.6 \\
G & 879$\,\pm\,$44 & 220$\,\pm\,$4 & 100$\,\pm\,$46 & 11.2$\,\pm\,$4.5 & 45.2$\,\pm\,$20.2 \\
H & 857$\,\pm\,$22 & 219$\,\pm\,$9 & 71$\,\pm\,$19 & 8.2$\,\pm\,$2.0 & 32.4$\,\pm\,$8.9 \\
I & 885$\,\pm\,$18 & 222$\,\pm\,$2 & 110$\,\pm\,$18 & 12.4$\,\pm\,$1.7 & 49.6$\,\pm\,$8.0 \\
J & 792$\,\pm\,$11 & 211$\,\pm\,$9 & 5$\,\pm\,$5 & 0.6$\,\pm\,$0.6 & 2.4$\,\pm\,$2.4 \\
K & 787$\,\pm\,$7 & 224$\,\pm\,$1 & 10$\,\pm\,$7 & 1.3$\,\pm\,$0.9 & 4.6$\,\pm\,$3.2 \\
L & 784$\,\pm\,$4 & 221$\,\pm\,$2 & 6$\,\pm\,$3 & 0.8$\,\pm\,$0.4 & 2.8$\,\pm\,$1.4 \\
M & 913$\,\pm\,$60 & 222$\,\pm\,$2 & 134$\,\pm\,$59 & 14.3$\,\pm\,$5.7 & 60.2$\,\pm\,$26.9 \\
N & 822$\,\pm\,$26 & 223$\,\pm\,$1 & 46$\,\pm\,$26 & 5.6$\,\pm\,$3.0 & 20.8$\,\pm\,$11.8 \\
O & 801$\,\pm\,$35 & 223$\,\pm\,$1 & 26$\,\pm\,$35 & 3.2$\,\pm\,$3.9 & 11.9$\,\pm\,$15.9 \\
P & 922$\,\pm\,$51 & 223$\,\pm\,$1 & 145$\,\pm\,$51 & 15.4$\,\pm\,$4.7 & 64.8$\,\pm\,$22.7 \\
Q & 815$\,\pm\,$15 & 218$\,\pm\,$7 & 33$\,\pm\,$17 & 4.1$\,\pm\,$2.0 & 15.2$\,\pm\,$7.3 \\
R & 802$\,\pm\,$34 & 220$\,\pm\,$8 & 23$\,\pm\,$27 & 2.7$\,\pm\,$2.9 & 10.7$\,\pm\,$13.5 \\
S & 880$\,\pm\,$51 & 224$\,\pm\,$1 & 106$\,\pm\,$51 & 11.8$\,\pm\,$5.0 & 47.5$\,\pm\,$22.8 \\
T & 796$\,\pm\,$7 & 223$\,\pm\,$1 & 16$\,\pm\,$7 & 2.1$\,\pm\,$0.8 & 7.4$\,\pm\,$3.0 \\
U & 818$\,\pm\,$60 & 220$\,\pm\,$7 & 33$\,\pm\,$55 & 3.6$\,\pm\,$5.7 & 15.6$\,\pm\,$26.9 \\
V & 791$\,\pm\,$4 & 223$\,\pm\,$1 & 8$\,\pm\,$4 & 1.0$\,\pm\,$0.5 & 3.6$\,\pm\,$1.9 \\
W & 807$\,\pm\,$5 & 223$\,\pm\,$1 & 28$\,\pm\,$5 & 3.4$\,\pm\,$0.6 & 12.4$\,\pm\,$2.4 \\
X & 804$\,\pm\,$8 & 223$\,\pm\,$1 & 23$\,\pm\,$8 & 2.8$\,\pm\,$1.0 & 10.2$\,\pm\,$3.7 \\
\bottomrule
\end{tabular}
\endgroup

%% file: tables/hyperparameters.tex
\label{tab:hyperparameters}
\scriptsize
\setlength{\tabcolsep}{4pt}
\renewcommand{\arraystretch}{0.85}
\begin{tabular}{lll}
\toprule
\textbf{Model family} & \textbf{Hyperparameter} & \textbf{Search values} \\
\midrule

\multirow{3}{*}{RF, RSF, ET, EST}
& \texttt{n_estimators}       & [100, 200, 500] \\
& \texttt{max_depth}          & [None, 5, 10, 15] \\
& \texttt{min_samples_split} & [5, 10, 20] \\
\midrule

\multirow{3}{*}{XGB-\{Reg, AFT, PH\}}
& \texttt{n_estimators}   & [128, 256, 512] \\
& \texttt{max_depth}      & [3, 6, 9] \\
& \texttt{learning_rate}  & [0.01, 0.05, 0.1] \\
\midrule

\multirow{3}{*}{TabM-\{Reg, AFT, PH\}}
& \texttt{K}  & [8, 16, 32] \\
& \texttt{d_blocks}           & [128, 256, 512] \\
& \texttt{n_blocks}  & [1, 2, 3] \\
\midrule

\multirow{2}{*}{LSTM-\{Reg, AFT, PH\}}
& \texttt{d_layer}           & [64, 128, 256] \\
& \texttt{n_layers}  & [1, 2, 3] \\
\midrule

\multirow{2}{*}{Transformer-\{Reg, AFT, PH\}}
& \texttt{d_model}           & [64, 128, 256] \\
& \texttt{n_blocks}  & [1, 2, 3] \\
\midrule

\{TabICLv2, TabPFNv2, Mitra, Tab-& 
\multirow{2}{*}{---} &
\multirow{2}{*}{---} \\
PFN-CF, TabSA, SurvivalPFN, SIC\} & & \\

\bottomrule
\end{tabular}

%% file: _05_results.tex
\section{Results and Discussion}\label{sec:results}

\begin{figure*}[t]
    \centering
    \includegraphics[width=\linewidth, trim=0 0.3cm 0 0, clip]{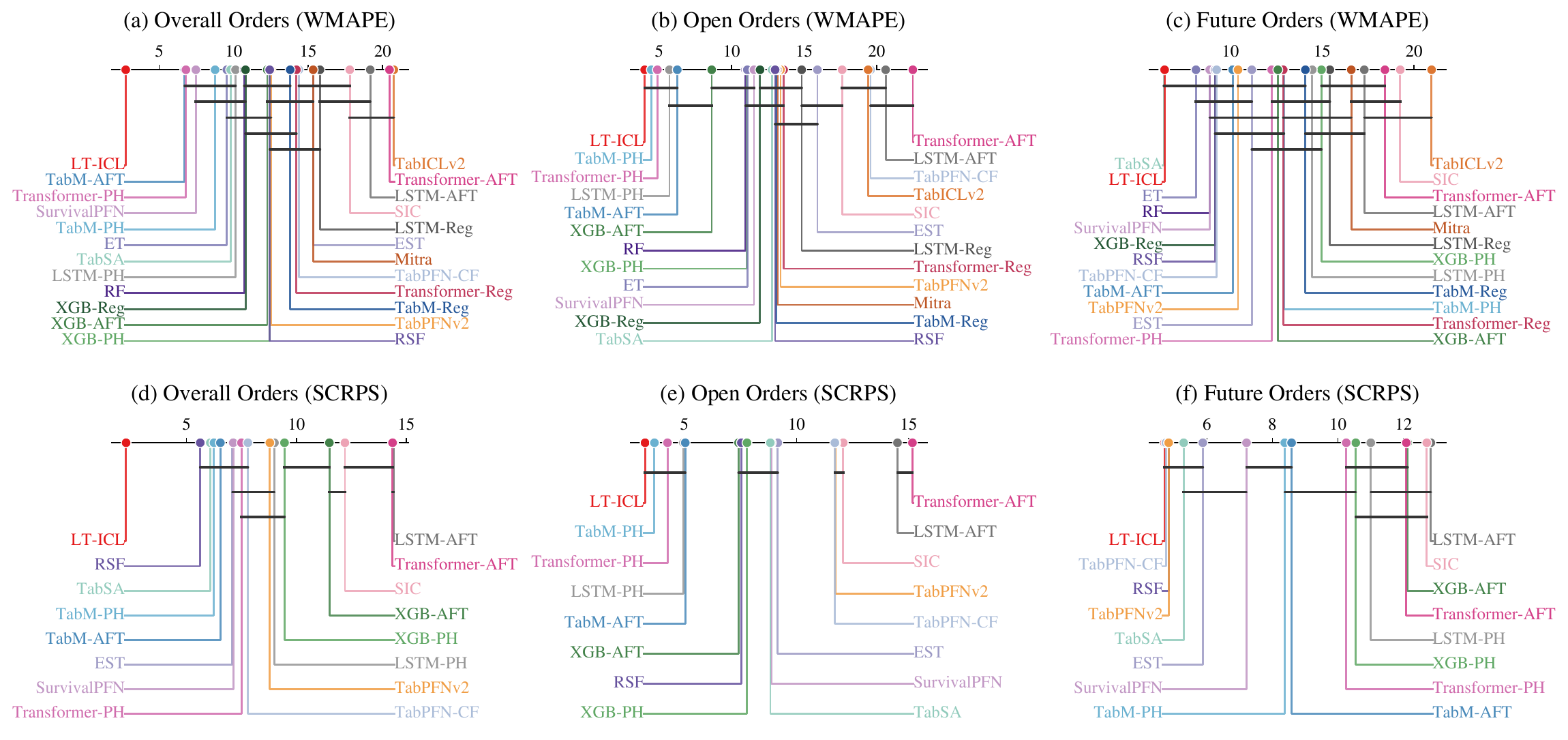}
    \caption{
        Critical difference diagrams comparing ranks across forecasting settings and metrics. Panels (a)--(c) report WMAPE ranks for overall, open-order, and future-order point forecasting, while panels (d)--(f) report SCRPS ranks for probabilistic forecasting. Lower ranks indicate better performance. Horizontal bars indicate groups of methods that are not significantly different under pairwise Wilcoxon signed-rank tests with Benjamini--Hochberg correction.
    }
    \label{fig:cd}
\end{figure*}

\subsection{Aggregate Forecasting Performance}

Across the 24 industrial lead time datasets, LT-ICL provides the strongest aggregate point-forecasting performance. As shown in \autoref{tab:wmape}, LT-ICL achieves the best overall WMAPE rank, the most dataset-level wins, the highest pairwise win rate, and the lowest average WMAPE. This performance is distributed across industries; LT-ICL obtains the lowest WMAPE on 1/1 automotive, 3/3 chemical, 2/5 electronics, 1/4 manufacturing, 2/3 medical, 4/7 pharmaceutical, and 0/1 semiconductor datasets. 
\rev{The strongest competing method is the  Transformer-PH, with the second-best rank, win rate and average WMAPE. However, it shares the second-highest number of wins with ET and TabM-PH.}{The strongest competing methods are TabM-AFT, Transformer-PH, and TabSA. TabM-AFT has the second-best rank and win rate, Transformer-PH the second-lowest average WMAPE, and TabSA the second-highest number of wins.} Thus, although several baselines are competitive under specific datasets, LT-ICL is the only method that leads across all aggregate point-forecasting criteria.

The probabilistic forecasting results show a similar pattern, but with stronger competition from survival-based baselines. As shown in \autoref{tab:scrps}, LT-ICL obtains the best overall SCRPS rank, the most dataset-level wins, the highest pairwise win rate, and the lowest average SCRPS. The industry-level win counts match those for WMAPE, except for one fewer win on the chemical datasets. RSF is the strongest competing probabilistic baseline by rank, win rate, and average SCRPS, while \add{TabSA achieves the second-highest number of dataset-level wins}. These results indicate that LT-ICL provides the strongest aggregate probabilistic performance, while tree-based survival models remain competitive distributional forecasters.

\begin{table}[t]
\centering
\scriptsize
\caption{\add{Scaled pinball loss at the 80th, 90th, and 95th percentiles, reported as mean $\pm$ standard deviation. Best values are bolded and second-best values are underlined.}}
\label{tab:spl}
\input{tables/spl}
\end{table}

\add{
While SCRPS evaluates predictive accuracy across the lead time distribution, supply planning decisions often depend on upper-tail quantiles. To assess this setting, \autoref{tab:spl} reports scaled pinball loss at the 80th, 90th, and 95th percentiles. LT-ICL achieves the second-lowest loss at the 80th percentile ($0.460$), trailing Transformer-PH ($0.450$). At the 90th percentile, LT-ICL remains competitive ($0.374$), behind Transformer-PH ($0.368$) and TabM-AFT ($0.372$). At the 95th percentile, LT-ICL achieves the lowest loss ($0.281$) among comparative methods. These results demonstrate that LT-ICL's competitive probabilistic forecasting performance extends to high lead time quantiles, which are particularly relevant for asymmetric-cost supply planning.
}

The critical difference diagrams in \autoref{fig:cd} provide a complementary rank-based view. LT-ICL obtains the best average rank in the overall WMAPE and SCRPS comparisons, confirming that its aggregate advantage is not only a consequence of scale differences in error values. Pairwise Wilcoxon signed-rank tests with Benjamini--Hochberg correction show that LT-ICL is statistically separated from competing methods in the overall forecasting setting for point and probabilistic forecasting. The benchmark contains heterogeneous procurement environments, and the strongest competing baseline varies across metrics, datasets, and industries. LT-ICL is therefore best characterized as the most robust overall forecasting method, rather than as a method that dominates every individual setting.

\begin{table*}[t]
\centering
\scriptsize
\caption{
Summary of point-forecasting performance across the overall, open-order, and future-order settings. Best values are shown in bold and second-best values are underlined.
 }
\label{tab:wmape}
\input{tables/wmape}
\end{table*}

\subsection{Open-Order and Future-Order Forecasting}
Separating open and future orders reveals how LT-ICL's aggregate performance is distributed across forecasting settings (see \autoref{fig:cd}, \autoref{tab:wmape}, and \autoref{tab:scrps}). For open order point forecasts, LT-ICL achieves the most wins, highest win rate and best rank, but has the second-best average WMAPE behind TabM-PH. \add{On future orders, LT-ICL achieves the most wins and ranks second to TabSA on average rank, pairwise win rate, and average WMAPE.}  \add{For probabilistic forecasting, LT-ICL leads in rank, wins, and win rate (tied with TabPFN-CF on future orders) across both settings. Its average SCRPS ranks second on open orders, behind TabM-PH, and fourth on future orders, behind TabPFN-CF, TabSA, and TabPFNv2. These findings highlight the competitiveness of in-context learning methods for future-order forecasting.}
 
These results suggest that LT-ICL's aggregate advantage is not driven by dominance in a single forecasting regime. Rather, the model performs consistently across both open-order and future-order settings. This is also reflected in the fact that the number of wins in each split is smaller than the total number of wins in the overall evaluation, indicating that LT-ICL's overall performance is accumulated across different forecasting targets. In contrast, the strongest competing baseline changes across settings and metrics: neural architectures are most competitive for open-order point forecasts, while \add{in-context learning methods} are particularly competitive for probabilistic future-order forecasts.

\begin{figure}[t]
    \centering
    \includegraphics[width=\linewidth, trim=0 0.3cm 0 0, clip]{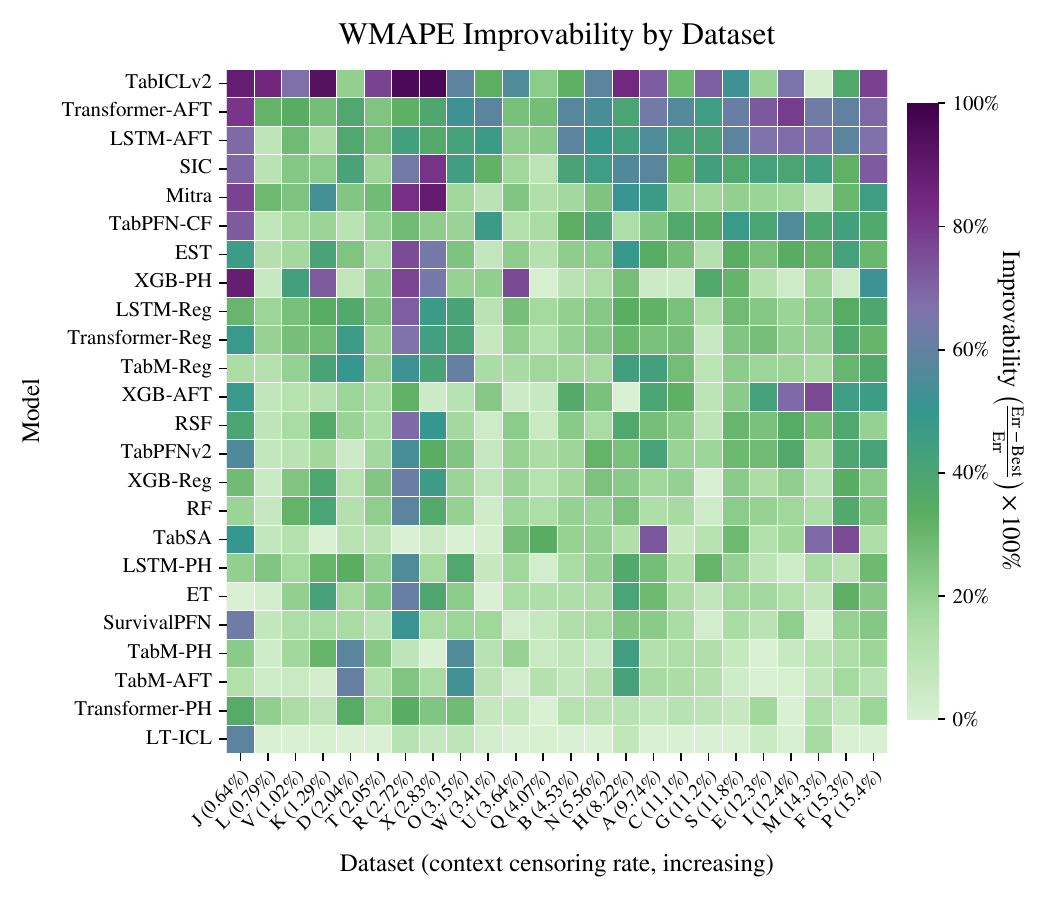}
    \caption{
    \add{
    Overall WMAPE improvability across datasets ordered by increasing context censoring rate. Each cell represents the percentage improvement required for a model to match the best-performing method on the corresponding dataset, following \cite{qu2026tabiclv2betterfasterscalable}. Lower values indicate better relative performance. 
    }
    } 
    \label{fig:improvement}
\end{figure}

\add{
We further examine whether forecasting performance varies with the prevalence of censored observations. \autoref{fig:improvement} reports overall WMAPE improvability, with datasets ordered by increasing context censoring rate. LT-ICL remains competitive across the observed range of censoring rates ($0.64\%$--$15.4\%$), exhibiting relatively small performance gaps from the best-performing method. In contrast, several baselines exhibit systematic changes in performance with increasing censoring rates. LSTM-AFT and Transformer-AFT show larger performance gaps on datasets with higher censoring rates, whereas SurvivalPFN tends to exhibit smaller gaps as censoring increases. These results suggest that LT-ICL maintains competitive relative performance across censoring rates, with its advantage not confined to a particular censoring regime.
}

\begin{table*}[t]
\centering
\scriptsize
\caption{
Summary of probabilistic forecasting performance across the overall, open-order, and future-order settings. Best values are shown in bold and second-best values are underlined.
}
\label{tab:scrps}
\input{tables/scrps}
\end{table*}

\subsection{Censoring Observation Ablation}
\begin{figure*}[!t]
    \centering
    \begin{minipage}[t]{\columnwidth}
        \centering
        \includegraphics[width=\linewidth, trim=0 0.3cm 0 0, clip]{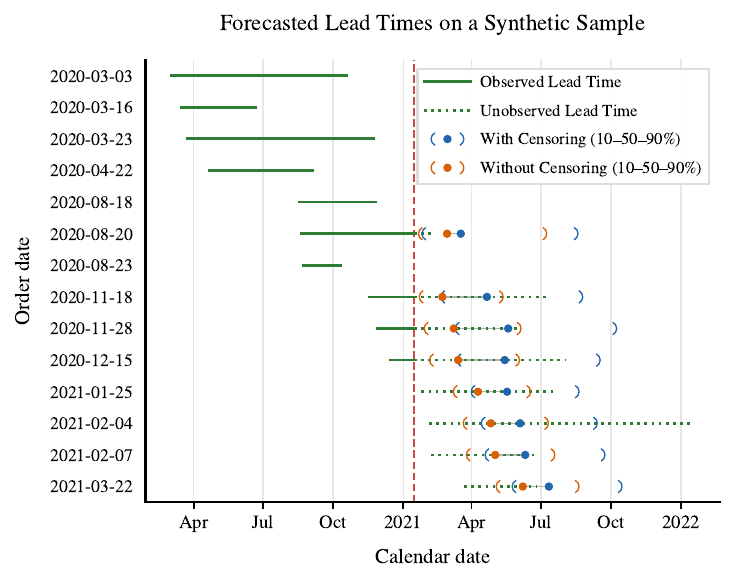}
        \caption{\add{
            Illustrative lead time forecasts with and without censored
            context observations on a synthetic sample drawn from the
            pretraining prior. The vertical dashed line marks the
            forecasting cutoff. Blue and orange markers denote the
            predicted 10th, 50th, and 90th percentiles of completion
            dates with and without censored observations, respectively.
        }}
        \label{fig:example}
    \end{minipage}%
    \hfill
    \begin{minipage}[t]{\columnwidth}
        \centering
        \includegraphics[width=\linewidth, trim=0 0.3cm 0 0, clip]{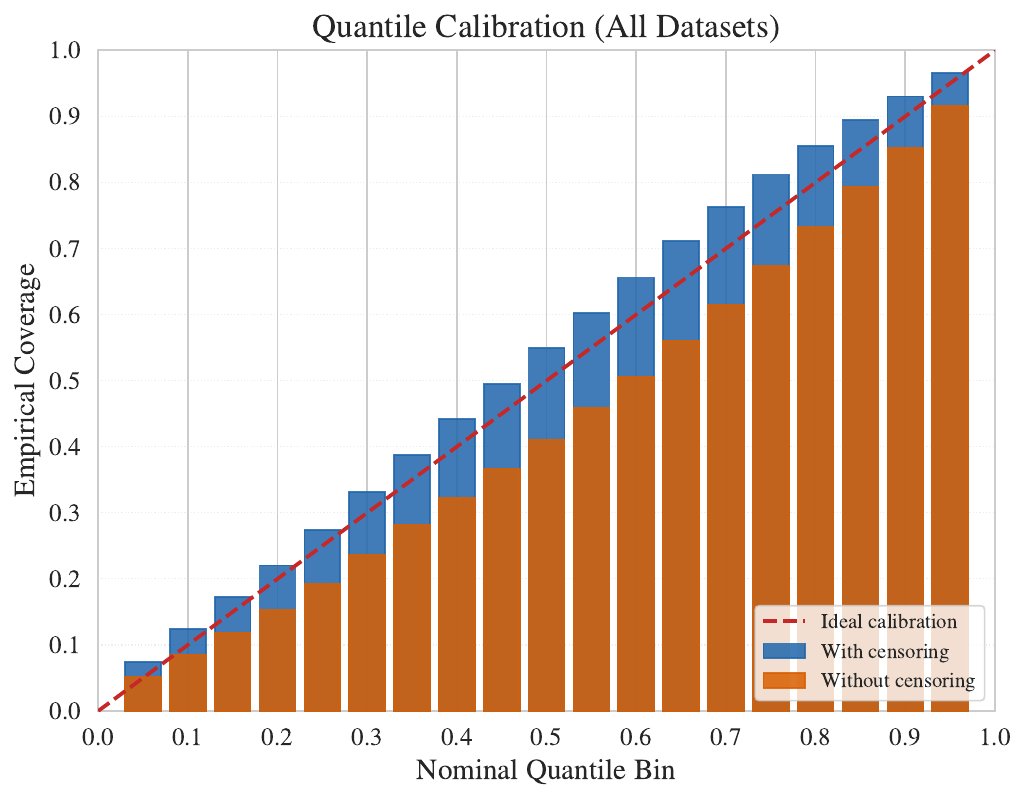}
        \caption{\add{
            Quantile calibration across all datasets with and without
            censored observations. Removing censored observations causes
            consistent under-forecasting, whereas including them causes
            slight over-forecasting.
        }}
        \label{fig:ece}
    \end{minipage}
\end{figure*}

We next isolate the contribution of censored open-order information by comparing LT-ICL with an ablation that removes open orders from the context. We use paired Wilcoxon signed-rank tests over dataset-level metric differences and report median effect sizes. Removing open orders from the context significantly worsens overall WMAPE, with a median increase of $0.044$ ($p=2.26\times 10^{-6}$). The effect is larger for open-order forecasts, where WMAPE increases by $0.089$ ($p=1.19\times 10^{-7}$). For future orders, WMAPE increases only slightly, with a median effect of $0.005$, and the difference is not statistically significant ($p=0.811$). The same pattern holds for probabilistic forecasting. Removing open orders significantly worsens overall SCRPS, with a median increase of $0.028$ ($p=2.01\times 10^{-5}$), and the effect is again larger for open-order forecasts, where SCRPS increases by $0.078$ ($p=1.19\times 10^{-7}$). For future-order forecasts, removing open orders slightly decreases SCRPS by $0.004$, but this effect is small and not statistically significant ($p=0.406$).

\add{
\autoref{fig:example} illustrates how censored information influences individual predictions. When censored observations are excluded, predictive medians shift toward earlier completion dates and prediction intervals become narrower. Incorporating censored observations generally produces later and wider forecasts, reflecting uncertainty about orders that have already remained unresolved for extended periods. \autoref{fig:ece} compares empirical coverage across nominal predictive quantiles, where perfect calibration corresponds to the diagonal reference line. Without censored observations, empirical coverage falls systematically below nominal levels, indicating that predicted quantiles tend to underestimate realized lead times. Incorporating censored observations largely corrects this tendency, although empirical coverage slightly exceeds nominal levels across much of the quantile range, indicating mild overestimation.
}

These results indicate that the benefit of censored observations is concentrated in the setting where censoring is directly relevant. Open orders provide lower-bound information about their realized lead times, and LT-ICL is able to use this partial information to improve forecasts for open orders. By contrast, future orders have no elapsed-duration information at prediction time, so censored context observations are less directly useful. The absence of a significant future-order improvement suggests that censored observations are not uniformly beneficial for all prediction targets; their value depends on whether the target setting can exploit elapsed-duration information. \add{The calibration results further show that excluding censored observations introduces systematic underprediction, which may be particularly consequential in supply chains with asymmetric costs.}

\begin{figure}[t]
    \centering
    \includegraphics[width=\linewidth, trim=0 0.3cm 0 0, clip]{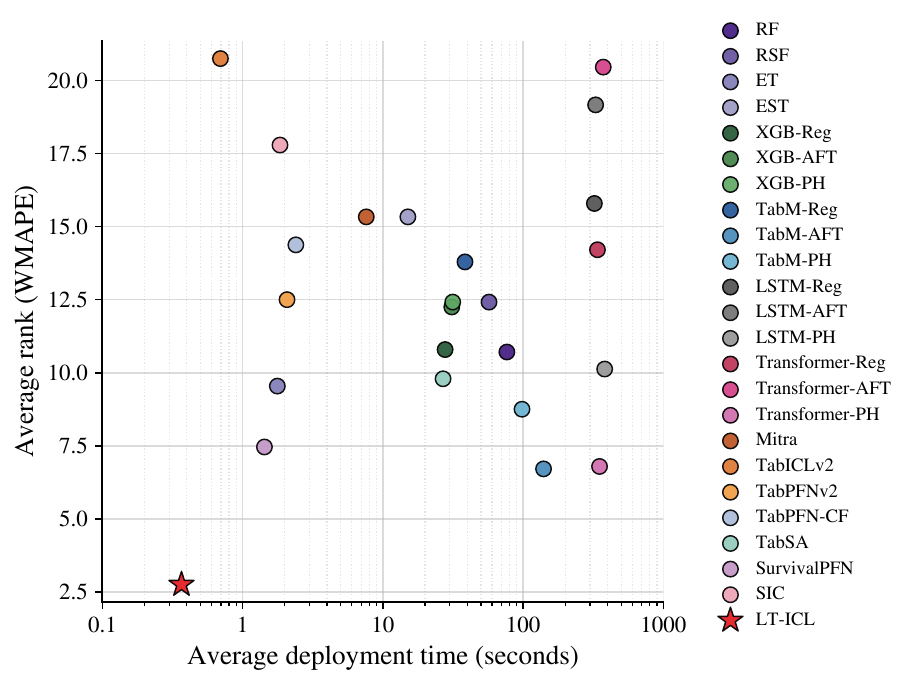}
    \caption{
        Average deployment-time runtime versus average WMAPE rank across benchmark datasets. Lower rank indicates better predictive performance. Runtime is shown on a log scale.
    }
    \label{fig:runtime}
\end{figure}

\subsection{Deployment-Time Efficiency}

\autoref{fig:runtime} compares deployment-time runtime with average rank under overall WMAPE. The practical advantage of LT-ICL is greatest in repeated-deployment settings, where a single pretrained model is reused across many lead time forecasting problems and the relevant cost is the marginal cost of adapting to each new dataset. LT-ICL achieves both the lowest average WMAPE rank and the lowest average deployment-time runtime. This places it in the lower-left region of \autoref{fig:runtime}, indicating that it improves the accuracy--runtime frontier rather than simply moving along it. Existing in-context learning baselines are also computationally efficient, but obtain \add{comparatively} weaker average ranks \add{(with the exception of SurvivalPFN)}. One possible explanation is that generic tabular \add{and survival analysis} in-context learning priors are less well aligned with industrial lead time data, where censoring, sparse part--site--supplier structure, and heavy-tailed lead time behavior are central. Among the supervised baselines, extra trees provide the strongest runtime--accuracy trade-off, while neural architectures tend to achieve competitive ranks \rev{only with substantially}{with substantially} higher deployment-time cost.

\begin{figure}[t]
    \centering
    \includegraphics[width=\linewidth, trim=0 0.375cm 0 0, clip]{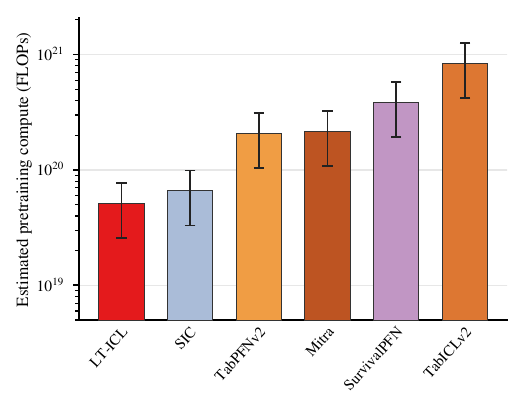}
    \caption{\add{Estimated pretraining compute (FLOPs, log scale) using the hardware-based approach of \cite{JRC143255}.  Bars assume 40\% GPU utilization, while whiskers indicate 20–60\% utilization.  TabPFN-CF and TabSA are excluded as separate pretraining runs because they both adapt a pretrained tabular foundation model.}}
    \label{fig:pretrain}
\end{figure}

\add{The deployment-time comparison excludes the one-time cost of pretraining. LT-ICL was pretrained for approximately nine hours using four NVIDIA H100 GPUs, after which the same model can be reused across downstream forecasting tasks. To contextualize this investment, \autoref{fig:pretrain} compares estimated pretraining compute across pretrained methods using the hardware-based approach of \cite{JRC143255}. Assuming 40\% GPU utilization, with an uncertainty range of 20--60\%, LT-ICL exhibits the lowest estimated pretraining compute among the methods considered. Therefore, LT-ICL's computational advantage is centered on both the deployment and pretraining settings, relative to its peers, for lead time forecasting. }


Finally, we examine whether task-specific fine-tuning can further trade runtime for predictive performance. Fine-tuning LT-ICL for five gradient steps using Schedule-Free AdamW \cite{defazio2024roadscheduled} with learning rate $\eta=5\times10^{-6}$ improves WMAPE by $0.038$ ($p=1.75\times10^{-4}$) and SCRPS by $0.022$ ($p=0.0126$). However, this increases deployment-time runtime by an average factor of $34.3\times$. These results show that task-specific optimization can further specialize the pretrained model, but \rev{by relinquishing the main computational advantage of in-context adaptation}{removes the main speed advantage of in-context prediction.} We therefore report the non-fine-tuned LT-ICL model as the primary method and view fine-tuning as an optional extension for applications where additional per-dataset optimization is acceptable.

%% file: tables/spl.tex
\renewcommand{\arraystretch}{0.85}
\begin{tabular}{@{}l|ccc@{}}
\toprule
Model & SPL@80 & SPL@90 & SPL@95 \\
\midrule
RSF
    & 0.638 $\pm$ 0.321
    & 0.559 $\pm$ 0.466
    & 0.424 $\pm$ 0.417 \\
EST
    & 0.854 $\pm$ 1.214
    & 0.641 $\pm$ 0.777
    & 0.454 $\pm$ 0.515 \\
XGB-AFT
    & 0.706 $\pm$ 0.149
    & 0.675 $\pm$ 0.148
    & 0.629 $\pm$ 0.153 \\
XGB-PH
    & 0.686 $\pm$ 0.513
    & 0.485 $\pm$ 0.305
    & 0.358 $\pm$ 0.207 \\
TabM-AFT
    & 0.482 $\pm$ 0.220
    & \underline{0.372 $\pm$ 0.171}
    & \underline{0.287 $\pm$ 0.142} \\
TabM-PH
    & 0.485 $\pm$ 0.207
    & 0.384 $\pm$ 0.173
    & 0.302 $\pm$ 0.149 \\
LSTM-AFT
    & 0.984 $\pm$ 0.246
    & 0.829 $\pm$ 0.254
    & 0.702 $\pm$ 0.249 \\
LSTM-PH
    & 0.484 $\pm$ 0.206
    & 0.383 $\pm$ 0.177
    & 0.305 $\pm$ 0.157 \\
Transformer-AFT
    & 1.030 $\pm$ 0.275
    & 0.880 $\pm$ 0.499
    & 0.633 $\pm$ 0.472 \\
Transformer-PH
    & \textbf{0.450 $\pm$ 0.177}
    & \textbf{0.368 $\pm$ 0.167}
    & 0.304 $\pm$ 0.159 \\
TabPFNv2
    & 0.614 $\pm$ 0.221
    & 0.549 $\pm$ 0.259
    & 0.455 $\pm$ 0.285 \\
TabPFN-CF
    & 0.580 $\pm$ 0.240
    & 0.519 $\pm$ 0.254
    & 0.466 $\pm$ 0.250 \\
TabSA
    & 0.578 $\pm$ 0.248
    & 0.470 $\pm$ 0.217
    & 0.373 $\pm$ 0.194 \\
SurvivalPFN
    & 0.640 $\pm$ 0.633
    & 0.576 $\pm$ 0.738
    & 0.451 $\pm$ 0.631 \\
SIC
    & 0.820 $\pm$ 0.621
    & 0.610 $\pm$ 0.502
    & 0.409 $\pm$ 0.356 \\
\midrule
\textbf{LT-ICL (Ours)}
    & \underline{0.460 $\pm$ 0.160}
    & 0.374 $\pm$ 0.144
    & \textbf{0.281 $\pm$ 0.118} \\
\bottomrule
\end{tabular}%

%% file: tables/wmape.tex
\renewcommand{\arraystretch}{0.85}
\begin{tabular}{l|rrrr|rrrr|rrrr}
\toprule
\textbf{Model}
& \multicolumn{4}{c|}{\textbf{Overall Orders}}
& \multicolumn{4}{c|}{\textbf{Open Orders}}
& \multicolumn{4}{c}{\textbf{Future Orders}} \\
\cmidrule(lr){2-5}
\cmidrule(lr){6-9}
\cmidrule(lr){10-13}
& Rank $\downarrow$ & Wins $\uparrow$ & WinRate $\uparrow$ & WMAPE $\downarrow$
& Rank $\downarrow$ & Wins $\uparrow$ & WinRate $\uparrow$ & WMAPE $\downarrow$
& Rank $\downarrow$ & Wins $\uparrow$ & WinRate $\uparrow$ & WMAPE $\downarrow$\\
\midrule

TabM-AFT
& \underline{6.71} & 1 & \underline{0.75} & 0.561
& 6.25 & 1 & 0.77 & 0.416
& 10.12 & 0 & 0.60 & 0.770 \\

Transformer-PH
& 6.79 & 2 & \underline{0.75} & \underline{0.531}
& 4.88 & 3 & 0.83 & 0.393
& 12.25 & 0 & 0.51 & 0.758 \\

SurvivalPFN
& 7.46 & 1 & 0.72 & 0.552
& 11.54 & 1 & 0.54 & 0.518
& 8.88 & 1 & 0.66 & 0.702 \\

TabM-PH
& 8.75 & 1 & 0.66 & 0.575
& \underline{4.46} & 2 & \underline{0.85} & \textbf{0.383}
& 12.92 & 0 & 0.48 & 0.798 \\

ET
& 9.54 & 2 & 0.63 & 0.607
& 11.08 & 0 & 0.56 & 0.495
& 8.12 & 3 & 0.69 & 0.766 \\

TabSA
& 9.79 & \underline{3} & 0.62 & 0.634
& 12.79 & 0 & 0.49 & 0.641
& \textbf{6.38} & \underline{4} & \textbf{0.77} & \textbf{0.583} \\

LSTM-PH
& 10.12 & 0 & 0.60 & 0.597
& 5.71 & \underline{4} & 0.80 & 0.393
& 14.46 & 0 & 0.41 & 0.867 \\

RF
& 10.71 & 0 & 0.58 & 0.603
& 10.96 & 0 & 0.57 & 0.489
& 8.88 & 1 & 0.66 & 0.714 \\

XGB-Reg
& 10.79 & 0 & 0.57 & 0.610
& 11.96 & 0 & 0.52 & 0.497
& 9.17 & 1 & 0.64 & 0.745 \\

XGB-AFT
& 12.25 & 1 & 0.51 & 0.627
& 8.62 & 3 & 0.67 & 0.484
& 12.58 & 0 & 0.50 & 0.730 \\

XGB-PH
& 12.42 & 0 & 0.50 & 0.865
& 11.04 & 0 & 0.56 & 0.757
& 14.96 & 0 & 0.39 & 1.278 \\

RSF
& 12.42 & 0 & 0.50 & 0.645
& 13.00 & 0 & 0.48 & 0.524
& 9.17 & 2 & 0.64 & 0.789 \\

TabPFNv2
& 12.50 & 0 & 0.50 & 0.607
& 13.38 & 0 & 0.46 & 0.509
& 10.42 & 0 & 0.59 & 0.726 \\

TabM-Reg
& 13.79 & 0 & 0.44 & 0.659
& 13.08 & 0 & 0.47 & 0.515
& 14.08 & 0 & 0.43 & 0.867 \\

Transformer-Reg
& 14.21 & 0 & 0.43 & 0.659
& 13.58 & 0 & 0.45 & 0.514
& 12.88 & 0 & 0.48 & 0.822 \\

TabPFN-CF
& 14.38 & 0 & 0.42 & 0.637
& 19.58 & 0 & 0.19 & 0.680
& 9.25 & 1 & 0.64 & 0.625 \\

EST
& 15.33 & 0 & 0.38 & 0.720
& 15.92 & 0 & 0.35 & 0.570
& 11.17 & 1 & 0.56 & 0.915 \\

Mitra
& 15.33 & 0 & 0.38 & 0.936
& 13.17 & 0 & 0.47 & 0.624
& 16.58 & 1 & 0.32 & 1.278 \\

LSTM-Reg
& 15.79 & 0 & 0.36 & 0.684
& 14.83 & 0 & 0.40 & 0.529
& 15.42 & 0 & 0.37 & 0.892 \\

SIC
& 17.79 & 0 & 0.27 & 0.843
& 17.62 & 0 & 0.28 & 0.724
& 19.25 & 0 & 0.21 & 1.162 \\

LSTM-AFT
& 19.17 & 0 & 0.21 & 0.810
& 20.62 & 1 & 0.15 & 0.785
& 17.29 & 0 & 0.29 & 0.874 \\

Transformer-AFT
& 20.46 & 0 & 0.15 & 0.896
& 22.50 & 0 & 0.07 & 0.893
& 18.42 & 0 & 0.24 & 0.899 \\

TabICLv2
& 20.75 & 0 & 0.14 & 2.834
& 19.42 & 0 & 0.20 & 1.259
& 20.96 & 0 & 0.13 & 4.218 \\

\midrule
\textbf{LT-ICL (Ours)}
& \textbf{2.75} & \textbf{13} & \textbf{0.92} & \textbf{0.467}
& \textbf{4.00} & \textbf{9} & \textbf{0.87} & \underline{0.388}
& \underline{6.42} & \textbf{9} & \underline{0.76} & \underline{0.619} \\

\bottomrule
\end{tabular}

%% file: tables/scrps.tex
\renewcommand{\arraystretch}{0.85}
\begin{tabular}{l|rrrr|rrrr|rrrr}
\toprule
\textbf{Model}
& \multicolumn{4}{c|}{\textbf{Overall Orders}}
& \multicolumn{4}{c|}{\textbf{Open Orders}}
& \multicolumn{4}{c}{\textbf{Future Orders}} \\
\cmidrule(lr){2-5}
\cmidrule(lr){6-9}
\cmidrule(lr){10-13}
& Rank $\downarrow$ & Wins $\uparrow$ & WinRate $\uparrow$ & SCRPS $\downarrow$
& Rank $\downarrow$ & Wins $\uparrow$ & WinRate $\uparrow$ & SCRPS $\downarrow$
& Rank $\downarrow$ & Wins $\uparrow$ & WinRate $\uparrow$ & SCRPS $\downarrow$ \\
\midrule

RSF
& \underline{5.62} & 1 & \underline{0.69} & \underline{0.376}
& 7.54 & 0 & 0.56 & 0.305
& 4.83 & 3 & \underline{0.74} & 0.507 \\

TabSA
& 6.08 & \underline{3} & 0.66 & 0.377
& 8.83 & 1 & 0.48 & 0.326
& 5.29 & 3 & 0.71 & \underline{0.459} \\

TabM-PH
& 6.25 & 2 & 0.65 & 0.403
& \underline{3.67} & 4 & \underline{0.82} & \textbf{0.250}
& 8.38 & 0 & 0.51 & 0.598 \\

TabM-AFT
& 6.54 & 0 & 0.63 & 0.412
& 5.04 & 0 & 0.73 & 0.270
& 8.58 & 0 & 0.49 & 0.628 \\

EST
& 7.08 & 2 & 0.59 & 0.403
& 9.17 & 0 & 0.46 & 0.326
& 5.88 & 2 & 0.68 & 0.569 \\

SurvivalPFN
& 7.12 & 1 & 0.59 & 0.433
& 8.88 & 1 & 0.47 & 0.333
& 7.21 & 1 & 0.59 & 0.652 \\

Transformer-PH
& 7.50 & 1 & 0.57 & 0.403
& 4.25 & \underline{5} & 0.78 & 0.267
& 10.25 & 0 & 0.38 & 0.625 \\

TabPFN-CF
& 7.79 & 0 & 0.55 & 0.403
& 11.71 & 0 & 0.29 & 0.387
& \underline{4.75} & \underline{4} & \textbf{0.75} & \textbf{0.449} \\

TabPFNv2
& 8.79 & 1 & 0.48 & 0.420
& 11.75 & 0 & 0.28 & 0.390
& 4.83 & 3 & \underline{0.74} & 0.469 \\

LSTM-PH
& 9.00 & 1 & 0.47 & 0.441
& 4.96 & 2 & 0.74 & 0.262
& 11.00 & 1 & 0.33 & 0.709 \\

XGB-PH
& 9.46 & 0 & 0.44 & 0.647
& 7.79 & 0 & 0.55 & 0.388
& 10.54 & 0 & 0.36 & 1.066 \\

XGB-AFT
& 11.50 & 0 & 0.30 & 0.476
& 7.42 & 1 & 0.57 & 0.307
& 12.12 & 0 & 0.26 & 0.661 \\

SIC
& 12.21 & 0 & 0.25 & 0.573
& 12.08 & 1 & 0.26 & 0.432
& 12.71 & 0 & 0.22 & 0.920 \\

Transformer-AFT
& 14.38 & 0 & 0.11 & 0.632
& 15.17 & 0 & 0.06 & 0.605
& 12.08 & 0 & 0.26 & 0.655 \\

LSTM-AFT
& 14.42 & 0 & 0.11 & 0.628
& 14.50 & 0 & 0.10 & 0.594
& 12.83 & 0 & 0.21 & 0.716 \\

\midrule
\textbf{LT-ICL (Ours)}
& \textbf{2.25} & \textbf{12} & \textbf{0.92} & \textbf{0.342}
& \textbf{3.25} & \textbf{9} & \textbf{0.85} & \underline{0.255}
& \textbf{4.71} & \textbf{7} & \textbf{0.75} & 0.495 \\

\bottomrule
\end{tabular}

%% file: _06_conclusion.tex
\section{Conclusion \& Future Work}\label{sec:conclusion}
This paper studied supplier lead time forecasting as a right-censored probabilistic prediction problem. To address this setting, we introduced LT-ICL, a censoring-aware in-context learning model for tabular lead time forecasting. By combining a transformer backbone with a conditional normalizing-flow head, LT-ICL produces a full predictive distribution over lead times.  The model is pretrained once on synthetic right-censored lead time tasks drawn from a hierarchical semi-parametric prior, then applied to new industrial datasets through in-context learning without task-specific parameter updates. Our excess-risk analysis \add{motivates} this design by bounding \rev{LT-ICL’s}{an ICL method's} forecasting error in terms of prior misspecification and amortized approximation error. This identifies two principled design levers through which LT-ICL can be improved: a richer synthetic prior that better represents real procurement systems, and a better-trained network that more accurately approximates the prior-induced posterior predictive distribution. 

The empirical results support the proposed formulation across 24 proprietary supply-chain datasets spanning seven industries. LT-ICL achieved the lowest overall WMAPE on 13 of 24 datasets and the lowest overall SCRPS on 12 datasets. The performance of our method is best explained as a balanced approach across both open order and future order forecasting, and does not uniformly dominate a single evaluation setting \add{or censoring rate}. We justify the censoring-aware approach through an ablation: including censored observations substantially improved open-order forecasting, while the corresponding effect on future orders was not statistically significant. This suggests that the benefit of the proposed approach comes from using partial lead time information in the setting where it is operationally relevant.

Several avenues for further extensions to LT-ICL remain. First, all datasets in our evaluation were restricted to a shared covariate schema to preserve comparability and confidentiality. \rev{LT-ICL could be extended to handle arbitrary covariates in a similar fashion to current tabular foundation models}{Consequently, the current model cannot directly exploit additional customer-specific covariates; a natural extension is support for arbitrary covariates}. \rev{Second, future work could study a planning-cost simulation or live deployment to complement the retrospective censoring used in our evaluation}{Second, although the present evaluation establishes improvements in point and probabilistic forecasting accuracy, it does not directly measure the downstream economic value of these improvements. Future work could integrate LT-ICL forecasts into MRP or inventory simulations to quantify effects on safety stock, service levels, stockouts, and planning costs, or evaluate these outcomes through live deployment to complement the retrospective censoring used in our evaluation}. Finally, LT-ICL is limited by the quadratic context cost of transformer architectures. Increasing the context size of our model through linear attention approximations remains a promising direction for further research. 

Overall, the results support right-censored probabilistic forecasting as a useful formulation for supplier lead time prediction. LT-ICL provides an initial demonstration that pretrained in-context models can combine censoring-aware inference, \add{domain-specific pretraining,} strong forecasting performance, and low task-specific adaptation cost in industrial lead time forecasting.